\documentclass{article}

\usepackage{arxiv}

\usepackage[utf8]{inputenc}
\usepackage[T1]{fontenc}
\usepackage{hyperref}
\usepackage{url}
\usepackage{booktabs}
\usepackage{amsfonts}
\usepackage{amsmath, amssymb}
\usepackage{amsthm}
\usepackage{nicefrac}
\usepackage{microtype}
\usepackage{lipsum}
\usepackage{graphicx}
\usepackage{caption, subcaption}
\usepackage{natbib}
\usepackage{tabularx}
\usepackage{multirow}
\usepackage{algpseudocode}
\usepackage{float}
\usepackage[linesnumbered,ruled,vlined]{algorithm2e}
\usepackage{amsmath, amssymb}
\usepackage{bm}
\newtheorem{proposition}{Proposition}
\newtheorem{theorem}{Theorem}
\newtheorem{lemma}{Lemma}

\title{VOLTA: The Surprising Ineffectiveness of Auxiliary Losses for Calibrated Deep Learning}

\author{
 Rahul D Ray \\
  Department of Electronics \& Electrical Engineering\\
  BITS Pilani, Hyderabad Campus.\\
  \texttt{f20242213@hyderabad.bits-pilani.ac.in} \\
   \And
  Utkarsh Srivastava \\
  Department of Economics \& Finance\\
  BITS Pilani, Hyderabad Campus\\
  \texttt{f20240633@hyderabad.bits-pilani.ac.in} \\
}

\begin{document}

\date{}
\maketitle
\thispagestyle{empty}

\begin{abstract}
Uncertainty quantification (UQ) is essential for deploying deep learning models in safety-critical applications, yet no consensus exists on which UQ method performs best across different data modalities and distribution shifts. This paper presents a comprehensive benchmark of ten widely used UQ baselines including MC Dropout, SWAG, ensemble methods, temperature scaling, energy-based OOD, Mahalanobis, hyperbolic classifiers, ENN, Taylor-Sensus, and split conformal prediction against a simplified yet highly effective variant of VOLTA that retains only a deep encoder, learnable prototypes, cross-entropy loss, and post-hoc temperature scaling. We evaluate all methods on CIFAR-10 (in-distribution), CIFAR-100, SVHN, uniform noise (out-of-distribution), CIFAR-10-C (corruptions), and Tiny ImageNet features (tabular). VOLTA achieves competitive or superior accuracy (up to 0.864 on CIFAR-10), significantly lower expected calibration error (0.010 vs. 0.044–0.102 for baselines), and strong OOD detection (AUROC 0.802). Statistical testing over three random seeds shows that VOLTA matches or outperforms most baselines, with ablation studies confirming the importance of adaptive temperature and deep encoders. Our results establish VOLTA as a lightweight, deterministic, and well-calibrated alternative to more complex UQ approaches.
\end{abstract}

\section{Introduction}

Deep neural networks (DNNs) have achieved remarkable success across computer vision, natural language processing, and healthcare. However, their deployment in safety-critical applications such as autonomous driving, medical diagnosis, and industrial control remains hindered by a fundamental limitation: classical deep learning models are often overconfident and fail to quantify their predictive uncertainty \cite{abdullah2022review, alqarafi2024estimating}. This lack of uncertainty awareness can lead to catastrophic failures when a model encounters inputs that deviate from its training distribution, a scenario known as out-of-distribution (OOD) data.

Uncertainty Quantification (UQ) has emerged as an essential research direction to address this limitation. By equipping DNNs with mechanisms to express when they are uncertain, UQ methods aim to build trustworthy, reliable, and safe AI systems \cite{he2026survey, tambon2022certify, wang2025uncertainty}. The literature broadly distinguishes two main types of uncertainty: \emph{aleatoric uncertainty}, which captures inherent noise in the data, and \emph{epistemic uncertainty}, which reflects model ignorance that can be reduced with more data \cite{fakour2024structured}. Effective UQ methods should not only provide accurate predictions but also produce well-calibrated confidence scores meaning that a model's predicted probability aligns with the true likelihood of correctness and enable robust detection of OOD samples \cite{vahdani2025towards, loftus2022uncertainty}.

A wide range of UQ techniques has been proposed in recent years. These include Bayesian approaches such as Monte Carlo (MC) Dropout and variational inference \cite{mena2021survey}, ensemble methods like Deep Ensembles and SWAG, as well as single‑pass deterministic methods that leverage feature distances, energy scores, or conformal prediction \cite{he2026survey, alqarafi2024estimating}. In parallel, the evaluation of OOD detection has matured through dedicated benchmarks, most notably OpenOOD \cite{yang2022openood, zhang2023openood}, which provides standardized protocols for comparing methods across anomaly detection, open‑set recognition, and model uncertainty. Despite these advances, several persistent challenges remain:

\begin{itemize}
    \item \textbf{Benchmark fragmentation:} Many studies evaluate OOD detection using only a two‑dataset scheme (e.g., training on CIFAR‑10 and testing on SVHN), which can produce biased and unreliable estimates of real‑world performance \cite{shafaei2018less, berger2021confidence}. A more robust evaluation requires multiple OOD datasets and a clear distinction between near‑OOD (perceptually similar but semantically different) and far‑OOD samples \cite{mukhoti2023raising}.
    \item \textbf{Calibration neglect:} While accuracy is widely reported, model calibration often measured by Expected Calibration Error (ECE) and Maximum Calibration Error (MCE) is equally critical for trustworthy decision‑making, yet frequently overlooked \cite{wang2023calibration, karimi2022improving}. Moreover, calibration performance can degrade significantly under distribution shifts \cite{wenzel2022assaying, kong2020calibrated}.
    \item \textbf{Incomplete comparisons:} Most comparative studies focus on a narrow subset of methods or datasets, leading to inconclusive or even contradictory findings \cite{manivannan2020comparative, wenzel2022assaying}. A systematic comparison across multiple data modalities (images, tabular features), multiple OOD scenarios, and multiple random seeds is still lacking.
\end{itemize}

To address these gaps, this paper presents an extensive empirical evaluation of ten representative UQ baselines alongside a novel simplified variant of the VOLTA framework (VOLTA). Our study is guided by the following key principles:

\begin{enumerate}
    \item \textbf{Comprehensive evaluation metrics:} Beyond accuracy, we measure calibration using ECE, MCE, Brier score, and Negative Log‑Likelihood (NLL). We assess selective prediction performance via Area Under the Risk‑Coverage curve (AURC) and Selective AUC. For OOD detection we report AUROC, AUPRC, and FPR@95TPR.
    \item \textbf{Diverse data modalities and distribution shifts:} Experiments cover image classification (CIFAR‑10, CIFAR‑100) and high‑dimensional tabular features (Tiny ImageNet features). OOD detection is evaluated against multiple semantically different datasets (CIFAR‑100 vs. CIFAR‑10, SVHN, uniform noise), and robustness is tested on the CIFAR‑10‑C corruption benchmark.
    \item \textbf{Statistical rigor:} All baselines and the proposed method are run across three random seeds; pairwise statistical significance is assessed via two‑sided t‑tests.
\end{enumerate}

Our contributions are threefold:

\begin{enumerate}
    \item We provide a systematic, large‑scale comparison of 10 established UQ methods MC Dropout, SWAG, Posterior Networks, Temperature Scaling, Energy‑based OOD, Mahalanobis, Hyperbolic Networks, ENN, Taylor‑Sensus, and Split Conformal on both vision and tabular tasks, using unified metrics and evaluation protocols.
    \item We introduce VOLTA (Simplified VOLTA), a lightweight deterministic approach that learns a normalized feature space with prototype‑based classification and a temperature‑scaled confidence estimate. VOLTA achieves competitive or superior calibration and OOD detection while requiring no stochastic sampling.
    \item We release all code, checkpoints, and results to facilitate reproducible research and provide a foundation for future benchmarking efforts.
\end{enumerate}

The remainder of this paper is organised as follows. Section ~\ref{sec:related} reviews related work on uncertainty quantification, calibration, and out-of-distribution detection. Section ~\ref{sec:preprocessing} introduces the data preprocessing pipeline. Section ~\ref{sec:method} describes the baseline methods and the proposed VOLTA framework, along with training protocols and evaluation metrics. Section ~\ref{sec:setup} details the experimental setup, including datasets, hyperparameters, and statistical procedures. Section ~\ref{sec:results} presents the main experimental results and analysis. Section ~\ref{sec:ablation} reports the ablation study. Section ~\ref{sec:discussion} discusses the findings, limitations, and future directions, and Section ~\ref{sec:conclusion} concludes the paper.

\section{Related Works}
\label{sec:related}

Uncertainty quantification (UQ) for deep neural networks has received substantial attention across multiple research communities, leading to a rich landscape of methods, evaluation protocols, and theoretical insights. This section reviews the most relevant literature, organising existing work into four interconnected themes: ensemble-based uncertainty estimation, out-of-distribution (OOD) detection with post-hoc scoring functions, calibration metrics and visualisation tools, and scalable Bayesian approximations via Laplace methods. Together, these bodies of work provide the methodological backdrop against which we position our proposed VOLTA framework.

\subsection{Ensemble-Based Uncertainty Estimation}

One of the most practical and widely adopted families of UQ methods builds on the idea of aggregating predictions from multiple models. Lakshminarayanan et al. \cite{lakshminarayanan2017simple} proposed deep ensembles, a simple non-Bayesian alternative that trains several neural networks with different random initialisations and combines their outputs. Their seminal work demonstrated that deep ensembles produce well-calibrated uncertainty estimates that are comparable or superior to those from approximate Bayesian neural networks, while significantly outperforming Monte Carlo (MC) dropout on out-of-distribution examples. The success of deep ensembles has motivated extensive comparative studies. Wu et al. \cite{wu2023quantification} interpret both deep ensembles and MC dropout as simple variants of Bayesian neural network inference, where the former aggregates several independently trained networks and the latter treats dropout as a variational approximation. Valdenegro-Toro and Mori \cite{valdenegro2022deeper} compared deep ensembles, MC dropout, Flipout and DropConnect on their ability to disentangle aleatoric and epistemic uncertainty, finding that ensembles provide the best overall disentanglement quality. Durasov et al. \cite{durasov2021masksembles} introduced Masksembles as a middle ground, creating a continuous spectrum of ensemble-like models that balance the computational cost of deep ensembles with the reliability of MC dropout.

MC dropout itself has been reinterpreted as an ensemble averaging strategy. Zhang et al. \cite{zhang2019confidence} attributed poor calibration of MC dropout to limited model diversity in the sampled ensemble and proposed structured dropout variants (dropBlock, dropChannel, dropLayer) to improve diversity and confidence calibration. Pop and Fulop \cite{pop2018deep} introduced Deep Ensemble Bayesian Active Learning (DEBAL), which uses an ensemble of MC dropout models with different initialisations to correct the mode collapse issue observed when a single MC dropout model is used for active learning. Bachstein \cite{bachstein2019uncertainty} proposed dropout ensembles that combine both approaches and evaluated them on regression tasks, noting that deep ensembles estimate both aleatory and epistemic uncertainty while MC dropout in its basic form provides only a single uncertainty entity.

Comparative studies have attempted to rank these methods. Manivannan \cite{manivannan2020comparative} empirically evaluated deep ensembles, MC dropout variants, and test-time data augmentation on CIFAR-10, CIFAR-100, and two custom datasets, providing a relative ranking based on uncertainty quality, classifier performance, and calibration. Liu et al. \cite{liu2021uncertainty} compared MC dropout and deep ensemble for a turbulence modelling application, finding that the deep ensemble produced smoother and more reasonable uncertainty estimates with good scalability. More recently, Chan et al. \cite{chan2024estimating} introduced HyperDM, a single-model method that estimates both epistemic and aleatoric uncertainty by hypernetwork-driven ensembling, achieving comparable accuracy to multi-model ensembles while significantly reducing training time compared to deep ensembles and outperforming MC dropout in prediction quality.

\subsection{Out-of-Distribution Detection with Post-Hoc Scores}

Detecting inputs that lie outside the training distribution is a critical capability for reliable deployment. A large thread of research focuses on post-hoc scoring functions that can be applied to any pre-trained classifier without retraining. The simplest and most widely used baseline is the maximum softmax probability (MSP) \cite{hendrycks2016baseline}. Hendrycks and Gimpel demonstrated that correctly classified examples typically have higher MSP than misclassified or OOD examples, and evaluated this baseline across computer vision, natural language processing, and speech recognition tasks. Despite its simplicity, MSP remains a reference point; however, Liu et al. \cite{liu2020energy} showed that softmax confidence can produce arbitrarily high values for OOD inputs, making it suboptimal. They proposed the energy score, derived from the free energy of a neural network, which is theoretically aligned with input probability density and less susceptible to overconfidence. They revealed that the log of the softmax confidence is mathematically equivalent to a special case of the free energy score after shifting logits by their maximum.

Subsequent work has sought to improve upon MSP while retaining its post-hoc nature. Liu et al. \cite{liu2023gen} proposed the Generalized ENtropy (GEN) score, an entropy-based function applicable to any pre-trained softmax classifier, which pushes the limits of softmax-only OOD detection and outperforms state-of-the-art post-hoc methods. Xia and Bouganis \cite{xia2022augmenting, xia2024augmenting} introduced Softmax Information Retaining Combination (SIRC) for selective classification with OOD data (SCOD). SIRC augments softmax-based confidence scores to improve separation of OOD data from correctly classified in-distribution samples, consistently outperforming or matching MSP across various OOD datasets and architectures. They also noted a trade-off: many OOD detection methods that improve over MSP sacrifice performance in distinguishing correctly from incorrectly classified in-distribution samples. Pearce et al. \cite{pearce2021understanding} examined the apparent contradiction that softmax confidence works moderately well despite common criticism, identifying two implicit biases approximately optimal decision boundary structure and filtering effects of deep networks that encourage MSP to correlate with epistemic uncertainty.

Other post-hoc methods leverage different sources of information. The Mahalanobis distance-based confidence score \cite{lee2018simple} has become a strong baseline. Kamoi and Kobayashi \cite{kamoi2020mahalanobis} investigated why Mahalanobis distance is effective for anomaly detection, finding that its superior performance stems from information not useful for classification, a different mechanism from prediction-confidence-based methods. They proposed combining Mahalanobis with ODIN to improve robustness. Denouden et al. \cite{denouden2018improving} combined Mahalanobis distance in the latent space of reconstruction autoencoders with the reconstruction loss, showing that the hybrid approach often improves OOD detection over reconstruction error alone. Zhang et al. \cite{zhang2025out} integrated an enhanced Mahalanobis distance with temperature scaling for power system text data, using multi-scaled PCA to enhance the distance computation.

Temperature scaling, originally proposed for calibration \cite{guo2017calibration}, has also proven effective for OOD detection. Mozafari et al. \cite{mozafari2019unsupervised} proposed unsupervised temperature scaling (UTS) that does not require labelled samples. Krumpl et al. \cite{krumpl2024ats} introduced Adaptive Temperature Scaling (ATS), a post-hoc method that dynamically computes a sample-specific temperature from intermediate layer activations, enhancing state-of-the-art OOD detectors. Roady et al. \cite{roady2019out} evaluated OOD detection methods on large-scale datasets (ImageNet-1K, Places-434), finding that ODIN (which combines temperature scaling and input perturbation) performs best for ImageNet, while OpenMax and Mahalanobis perform best for Places. Hendrycks et al. \cite{hendrycks2019scaling} showed that MSP does not scale well to large-scale, challenging conditions, and found that a maximum logit (MaxLogit) baseline consistently outperforms MSP. Nitsch et al. \cite{nitsch2021out} applied OOD detection to automotive perception, noting that neural networks often assign high softmax confidence to OOD samples, and proposed a method that outperforms state-of-the-art on real-world automotive datasets without requiring OOD data during training. Bulusu et al. \cite{bulusu2020anomalous} surveyed anomalous example detection methods, noting that combining softmax prediction score and entropy with temperature scaling yields a score that is high for in-distribution examples and low for OOD examples. Bui and Liu \cite{bui2023density} proposed Density-Softmax, a deterministic framework that combines a density function with the softmax layer to reduce overconfidence under distribution shifts; it is distance-aware in feature space, leading to uniform probabilities and high uncertainty on OOD data. Veličković et al. \cite{velivckovic2024softmax} demonstrated a fundamental limitation of softmax to robustly approximate sharp functions out-of-distribution, proposing adaptive temperature as an ad-hoc remedy.

\subsection{Calibration Metrics and Reliability Diagrams}

Evaluating the calibration of probabilistic predictions is essential for comparing UQ methods. Guo et al. \cite{guo2017calibration} popularised the use of reliability diagrams and the expected calibration error (ECE) for modern neural networks, showing that temperature scaling is a simple and effective post-hoc calibration technique. Reliability diagrams plot expected accuracy as a function of confidence; deviations from the identity line indicate miscalibration. Vaicenavicius et al. \cite{vaicenavicius2019evaluating} developed a theoretical framework for calibration evaluation, proposing multidimensional reliability diagrams for multiclass settings. Arrieta-Ibarra et al. \cite{arrieta2022metrics} compared histogram-based reliability diagrams with a cumulative approach that avoids bin width choices, introducing empirical cumulative calibration errors (ECCEs) that offer statistical advantages over ECE.

Lane \cite{lane2025comprehensive} provided a comprehensive review of classifier probability calibration metrics, clarifying that while a bar chart is often called a reliability diagram, a line representation is sometimes distinguished as a reliability plot. Fan et al. \cite{fan2025calzone} introduced Calzone, a Python package that provides reliability diagrams with error bars, class-conditional calibration error, and different binning schemes, addressing limitations of existing libraries. Minderer et al. \cite{minderer2021revisiting} used reliability diagrams to compare modern neural network architectures, showing that MLP-Mixer and Vision Transformers combine high accuracy with good calibration. Kumar et al. \cite{kumar2019verified} evaluated recalibration methods including Platt scaling and temperature scaling, finding them less calibrated than reported, and introduced a scaling-binning calibrator. Gupta et al. \cite{gupta2020calibration} proposed a binning-free calibration metric (KS error) inspired by the Kolmogorov–Smirnov test, and developed spline-based recalibration. Dheur and Taieb \cite{dheur2023large} used PIT (probability integral transform) reliability diagrams for regression calibration. Wang et al. \cite{wang2021rethinking} used reliability diagrams to visualise the calibration effect of temperature scaling and argued that overconfidence is not always detrimental.

\subsection{Scalable Bayesian Inference via Laplace Approximation}

The Laplace approximation has recently emerged as a powerful post-hoc tool for turning pre-trained neural networks into Bayesian neural networks. Daxberger et al. \cite{daxberger2021laplace} demonstrated through extensive experiments that the Laplace approximation is competitive with more popular alternatives while excelling in computational cost. They highlighted that last-layer Laplace approximation has minimal overhead over a standard forward pass and released the laplace library for scalable last-layer LAs. Deng et al. \cite{deng2022accelerated} proposed the accelerated linearized Laplace approximation (ELLA) to address unreliability issues of existing approximations, showing that ELLA provides competitive predictive uncertainty and scales to vision transformers. Ortega et al. \cite{ortega2026scalable} introduced ScaLLA, which approximates the kernel of the linearized Laplace approximation using a surrogate neural network to avoid explicit Jacobian computation, achieving scalability via efficient Jacobian-vector products. Ortega et al. \cite{ortega2023variational} proposed variational linearized Laplace approximation (VaLLA), which achieves training cost sub-linear in dataset size by stochastic optimisation, outperforming last-layer and Kronecker factorised approximations in predictive distribution quality.

Bergamin et al. \cite{bergamin2023riemannian} developed Riemannian Laplace approximations that respect the geometry of the parameter space, noting that the Laplace approximation has become popular partly due to new scalable computation methods, and that last-layer approaches are currently the only competitive methods against ensembles. Erick et al. \cite{erick2025last} used last-layer Laplacian approximation for medical image analysis, finding that restricting stochasticity to the last layer sufficiently provides competitive uncertainty estimation. Weber et al. \cite{weber2025laplax} implemented Laplace approximations in JAX, addressing last-layer and subnetwork approximations. Perone et al. \cite{perone2021l2m} proposed L2M, which constructs the Laplace approximation from the gradient second moment already estimated by optimisers like Adam, requiring no extra computation. Gui et al. \cite{gui2021laplace} studied diagonal Hessian approximation within the Laplace framework, demonstrating computational convenience and overcoming ill-posedness in large-scale problems. Huseljic et al. \cite{huseljic2022efficient} proposed a Bayesian update method using last-layer Laplace approximation to avoid costly retraining when new data arrives, computing the inverse Hessian in closed form for low complexity.

In summary, the literature provides a rich toolkit for uncertainty quantification, from ensemble methods and post-hoc OOD scores to calibration metrics and scalable Bayesian approximations. Our proposed VOLTA framework builds on these insights by combining a deep normalised encoder with prototype-based classification and temperature scaling, offering a deterministic, well-calibrated alternative that is simpler than ensembles and more scalable than full Laplace approximations while achieving competitive or superior performance across multiple metrics.

\section{Data preprocessing}
\label{sec:preprocessing}

All datasets were preprocessed following a fixed, meticulously documented pipeline to ensure reproducibility and strict fairness when comparing across different baseline models. The preprocessing steps differed between vision and tabular data but shared common elements: deterministic train–validation splits, input normalisation, and, where applicable, corruption synthesis or feature extraction. The following subsections detail the treatment of each dataset, and Table~\ref{tab:preprocessingtable} summarises all preprocessing choices in a compact form.

\subsection{Vision Datasets}

\textbf{CIFAR-10 (in‑distribution).}  
CIFAR‑10 consists of 50\,000 training and 10\,000 test colour images of size $32\times32$ pixels with ten classes. To obtain a clean validation set for hyperparameter tuning and early stopping, the training set was further split into a training subset of 40\,000 images and a validation subset of 10\,000 images using a fixed random seed (42). During training, standard data augmentation was applied to improve generalisation: random horizontal flipping with a probability of 0.5 and random cropping after adding four pixels of zero padding on each side (i.e., the image was first padded to $40\times40$, then a random $32\times32$ crop was taken). All images were normalised channel‑wise using the per‑channel mean $(0.5,0.5,0.5)$ and standard deviation $(0.5,0.5,0.5)$ to map pixel values from the original $[0,1]$ range to the $[-1,1]$ interval. Test images were only normalised without any augmentation. This preprocessing is identical to that used in many benchmark studies, enabling direct comparison with prior work.

\textbf{CIFAR‑100 (out‑of‑distribution for CIFAR‑10, in‑distribution for its own experiments).}  
CIFAR‑100 contains 50\,000 training and 10\,000 test images of size $32\times32$ pixels, but with 100 fine‑grained classes. When used as an out‑of‑distribution dataset for models trained on CIFAR‑10, no data split was necessary; the entire test set (10\,000 images) was employed as OOD data. For experiments where CIFAR‑100 served as an in‑distribution dataset (e.g., when evaluating a model's ability to learn fine‑grained categories), the training set was split into 45\,000 training and 5\,000 validation images, again using a fixed random seed. CIFAR‑100 was normalised using its own per‑channel statistics: mean $(0.5071,0.4867,0.4408)$ and standard deviation $(0.2675,0.2565,0.2761)$. Unlike CIFAR‑10, no data augmentation was applied to CIFAR‑100 because the primary goal was to measure clean performance on a more challenging class hierarchy without the confounding effect of augmentation.

\textbf{SVHN (out‑of‑distribution).}  
The Street View House Numbers (SVHN) dataset provides 26\,032 test images of size $32\times32$ pixels, each containing a digit from 0 to 9. SVHN was used exclusively as an out‑of‑distribution dataset for models trained on CIFAR‑10. To align with the input distribution expected by those models, SVHN images were normalised using exactly the same CIFAR‑10 mean and standard deviation $(0.5,0.5,0.5)$ and $(0.5,0.5,0.5)$. No resizing or cropping was required because SVHN images already match the $32\times32$ resolution of CIFAR‑10. No augmentation was applied.

\textbf{Uniform Noise (synthetic out‑of‑distribution).}  
To test model behaviour under extreme distributional shift, a synthetic out‑of‑distribution set was generated. We sampled 5\,000 independent images from a uniform distribution $\mathcal{U}[0,1]$, each with shape $3\times32\times32$ (i.e., three colour channels, 32 pixels height and width). These random tensors were then normalised using the same CIFAR‑10 statistics (mean 0.5, standard deviation 0.5) so that they lie approximately in the same numerical range as the in‑domain data, albeit without any natural image structure. This synthetic OOD set serves as a challenging stress test for any classifier.

\textbf{CIFAR‑10‑C (corruptions).}  
CIFAR‑10‑C is a benchmark dataset that applies 19 common image corruptions (e.g., Gaussian noise, shot noise, impulse noise, defocus blur, motion blur, zoom blur, snow, frost, fog, brightness, contrast, elastic transform, pixelate, JPEG compression, etc.) to every image in the original CIFAR‑10 test set. Each corruption is applied at five severity levels. For computational efficiency and to maintain a balanced evaluation, we retained only the first 10\,000 corrupted images per corruption type (i.e., per corruption–severity combination). All corrupted images were normalised with the same CIFAR‑10 mean and standard deviation used for clean images, ensuring that the only difference between clean and corrupted inputs is the presence of the corruption, not a change in normalisation.

\textbf{ResNet‑18 Feature Extraction (for VOLTA experiments).}  
The VOLTA model (a variant of the VOLTA framework) operates not on raw pixels but on features extracted by a fixed, pre‑trained ResNet‑18. For experiments on CIFAR‑10 and CIFAR‑100, we employed a ResNet‑18 model pre‑trained on ImageNet (1.28\,million images, 1000 classes) as a frozen feature extractor. The preprocessing pipeline for this extractor involved several steps: first, the input image was denormalised from the dataset‑specific statistics back to the raw $[0,1]$ range; second, the image was resized to $224\times224$ pixels using bilinear interpolation to match the input size expected by ResNet‑18; third, the resized image was normalised using ImageNet statistics (mean $[0.485,0.456,0.406]$, standard deviation $[0.229,0.224,0.225]$); finally, the preprocessed image was passed through the ResNet‑18 model (without the final classification layer) to obtain a 512‑dimensional feature vector. Because this extraction is computationally expensive, we cached all feature vectors to disk after the first run, so that subsequent experiments (e.g., with different random seeds or hyperparameters) could load the cached features directly, guaranteeing both speed and reproducibility.

\subsection{Tabular Dataset}

\textbf{Tiny ImageNet Features.}  
This dataset consists of pre‑computed feature vectors extracted from an ensemble model trained on Tiny ImageNet. Each feature vector has 512 dimensions. The original 200 classes were split functionally: the first 100 classes (indices 0–99) formed the in‑distribution set, while the remaining 100 classes (100–199) served as the out‑of‑distribution set. From the in‑distribution training features, we initially had 80\,000 samples after filtering out any missing or corrupted entries. A fixed random seed was used to hold out 20\% of these samples as a validation set, resulting in 64\,000 training samples and 16\,000 validation samples. The test set for the in‑distribution part comprised 5\,000 samples from classes 0–99 (50 per class), and the OOD test set comprised 5\,000 samples from classes 100–199 (again 50 per class). No further normalisation or scaling was applied to the features because they already lie in a well‑behaved range (the ensemble's final layer activations are bounded). This tabular dataset allows us to evaluate uncertainty quantification methods on a high‑dimensional, non‑image modality.

\subsection{Summary of Preprocessing Steps}

Table~\ref{tab:preprocessingtable} provides a concise overview of all preprocessing operations for each dataset used in this study. To fit the width of a standard academic paper column, the table uses abbreviated column headers and compact entries. The columns are: Dataset (name), Role (in‑distribution ID or out‑of‑distribution OOD), Augmentation (type applied during training), Normalisation (mean and std used), Splits (train/validation/test counts), OOD creation (method if synthetic), and Special notes (any additional remarks). All preprocessing code was implemented in Python using standard libraries (NumPy, scikit‑learn, torchvision) and applied identically to all compared baselines to ensure a fair evaluation.

\begin{table}[htbp]
\centering
\caption{Summary of preprocessing steps for all datasets. Norm.: normalisation; OOD: out-of-distribution; ID: in-distribution. Important thing to note that SVHN, Uniform Noise and CIFAR-10-C do not have any training or validation split. SVHN was used exclusively as out‑of‑distribution (OOD) test data. No training or validation split, CIFAR-10-C was used as corrupted test set. No training or validation split..}
\label{tab:preprocessingtable}
\footnotesize
\begin{tabular}{@{}lcll@{}}
\toprule
\textbf{Dataset} & \textbf{Role} & \textbf{Norm. (mean/std)} & \textbf{Splits (tr/val/te)} \\
\midrule
CIFAR-10 & ID & (0.5,0.5,0.5)/(0.5,0.5,0.5) & 40k/10k/10k \\
CIFAR-100 & OOD (for C10); ID & (0.507,0.487,0.441)/(0.268,0.257,0.276) & 45k/5k/10k \\
SVHN & OOD & same as CIFAR-10 & –/–/26k \\
Uniform Noise & OOD & same as CIFAR-10 & –/–/5k \\
CIFAR-10-C & corruption & same as CIFAR-10 & –/–/10k per corr. \\
TinyImageNetFeat & ID+OOD & none & 64k/16k/5k (ID); 5k OOD \\
ResNet-18 Feat & ID (VOLTA) & ImageNet stats & same as source \\
\bottomrule
\end{tabular}
\end{table}

\section{Methodology}
\label{sec:method}

This section details the experimental framework, including the neural architectures, the ten uncertainty quantification (UQ) baselines, the proposed VOLTA method, training protocols, evaluation metrics, and the procedures for statistical comparison and ablation. All experiments were implemented in PyTorch and executed on an NVIDIA Tesla T4 GPU.

\subsection{Model Architectures}

For the vision experiments on CIFAR-10 and CIFAR-100, a common convolutional neural network (CNN) backbone was employed. This backbone consists of two convolutional blocks. The first block contains two $3\times3$ convolutional layers with 64 filters, each followed by batch normalisation and a ReLU activation, and a $2\times2$ max-pooling layer with stride 2. Dropout (with rate 0.2 when used for MC Dropout, otherwise 0.0) is applied after the pooling layer. The second block is identical except that it uses 128 filters. The resulting feature map of size $128\times8\times8$ is flattened to 8192 dimensions and passed through a fully connected layer with 256 units and ReLU, followed by a dropout layer and a final linear layer that maps to the number of classes (10 for CIFAR-10, 100 for CIFAR-100). For tabular experiments on the Tiny ImageNet features dataset (512‑dimensional input vectors), a three‑layer multilayer perceptron (MLP) was used. The MLP has hidden layers of sizes 256, 128, and 64, each with ReLU activation and dropout (rate 0.2). The final 64‑dimensional feature vector is passed to a linear classifier with 100 output units (the number of in‑distribution classes). This same CNN and MLP architectures serve as the shared base for all post‑hoc baselines, including temperature scaling, energy‑based OOD detection, Mahalanobis distance, and split conformal prediction.

Several baselines required specialised architectures. For hyperbolic classifiers, the Euclidean output layer is replaced with a hyperbolic head. Given the $L_2$‑normalised feature vector $\mathbf{z}$ and class prototypes $\mathbf{p}_k$ in the Poincaré ball of curvature $c=1.0$, the logit for class $k$ is defined as the negative squared distance $-\|\mathrm{proj}(\mathbf{z}) \ominus \mathrm{proj}(\mathbf{p}_k)\|^2$, where $\mathrm{proj}$ is the exponential map from the tangent space to the ball. For the ENN (Epistemic Neural Network), we augment the model with a latent variable $z$ drawn uniformly from $\{1,\dots,8\}$. The network consists of the same feature extractor as the base CNN, followed by an embedding layer for $z$ and a small MLP that concatenates the feature vector with the embedding to produce logits. During training, a random $z$ is sampled per sample; during inference, the predictive distribution is obtained by averaging the softmax outputs over all eight values of $z$. For SWAG (Stochastic Weight Averaging–Gaussian), we train a deterministic model (no dropout) with standard cross‑entropy, then collect weight snapshots for 20 additional epochs using SGD with learning rate $10^{-3}$ and momentum 0.9. The first two moments (mean and a low‑rank covariance of rank 20) are estimated from these snapshots. At test time, 10 weight samples are drawn from the resulting Gaussian distribution, and the predictive probabilities are averaged.

\begin{figure*}[t]
    \centering
    \includegraphics[width=0.95\textwidth]{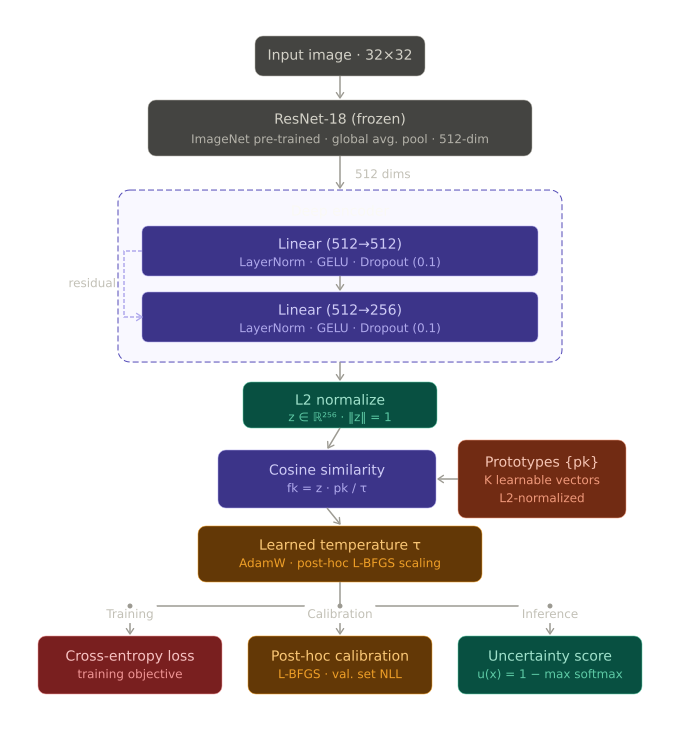}
    \caption{
    \textbf{VOLTA Architecture Overview.} 
    The model consists of a frozen feature extractor (ResNet-18 for vision or raw feature input for tabular data), 
    followed by a deep encoder that maps inputs into a normalized latent space. Learnable class prototypes are 
    embedded in the same space, and classification is performed via cosine similarity between features and prototypes. 
    A learnable temperature parameter scales the logits, followed by post-hoc temperature calibration. 
    The resulting framework enables efficient, deterministic uncertainty estimation without stochastic sampling.
    }
    \label{fig:volta_architecture}
\end{figure*}

\subsection{Uncertainty Quantification Baselines}

We implemented ten baseline methods that cover a broad spectrum of uncertainty quantification approaches: approximate Bayesian inference (MC Dropout, SWAG), post‑hoc calibration (Temperature Scaling, Split Conformal), deterministic uncertainty (Energy‑based OOD, Mahalanobis, Posterior Networks, Hyperbolic, ENN, Taylor‑Sensus). All baselines share the same data preprocessing and evaluation pipeline to ensure fair comparison.

MC Dropout uses dropout (rate 0.2) both during training and inference. At test time, ten forward passes with dropout enabled are performed; the predictive probability is the mean of the softmax outputs, and predictive entropy serves as the uncertainty score. SWAG, as described above, also averages ten weight samples. Posterior Networks first train a deterministic CNN with cross‑entropy. On the training set, we extract the penultimate‑layer features and fit a Gaussian mixture model (GMM) with full covariance matrices. For CIFAR‑10 we use ten components, for CIFAR‑100 and Tiny ImageNet we use 50 components (to limit computational cost). The negative log‑likelihood of the test point under the GMM is used as the OOD score. Temperature Scaling calibrates a pre‑trained deterministic CNN by optimising a single scalar temperature $T>0$ on the validation set to minimise the negative log‑likelihood of the logits, using the L‑BFGS optimiser. The energy‑based method uses the same deterministic model and computes the energy score $E(\mathbf{x}) = -\log\sum_k e^{f_k(\mathbf{x})}$, where $f_k$ are the logits. Lower energy indicates higher confidence.

The Mahalanobis baseline fits class‑conditional Gaussian distributions on the penultimate‑layer features of the training set. For each class $c$, we compute the mean feature vector $\boldsymbol{\mu}_c$ and a shared precision matrix $\boldsymbol{\Lambda}$ as the pseudo‑inverse of the pooled covariance matrix (regularised with $10^{-6}$ on the diagonal). The OOD score is the minimum Mahalanobis distance $\min_c (\mathbf{f} - \boldsymbol{\mu}_c)^\top \boldsymbol{\Lambda} (\mathbf{f} - \boldsymbol{\mu}_c)$ across all classes. The hyperbolic classifier is trained with cross‑entropy on the Poincaré logits; its uncertainty is measured as the entropy of the softmax output, while OOD detection uses the inverse of the margin between the top two logits. The ENN, described earlier, is trained with cross‑entropy where each sample is paired with a random $z$. Its predictive uncertainty is the entropy of the averaged softmax over $z$. Taylor‑Sensus adds a regularisation term to the cross‑entropy loss: $0.1 \cdot \mathrm{Var}(\mathrm{logits})$, encouraging the logits to have low variance. This method remains deterministic and uses predictive entropy as uncertainty. Finally, split conformal prediction uses the validation set to compute non‑conformity scores $s_i = 1 - \hat{p}_{y_i}(\mathbf{x}_i)$, where $\hat{p}_{y_i}$ is the predicted probability for the true class. The threshold $\hat{q}$ is set to the $\lceil (n+1)(1-\alpha) \rceil / n$ quantile of the scores with $\alpha = 0.1$. The size of the prediction set (or the indicator of whether the set is non‑empty) serves as an uncertainty proxy.

\subsection{Proposed Method: VOLTA}

The proposed VOLTA method is a lightweight, deterministic approach that combines a deep encoder, learnable prototypes, and temperature scaling. Unlike the full VOLTA framework, this simplified version omits reconstruction, margin, diversity, and contrastive losses, retaining only cross‑entropy, which makes it computationally efficient and easy to train.

For vision datasets (CIFAR‑10 and CIFAR‑100), we use a fixed ResNet‑18 pre‑trained on ImageNet as a feature extractor. Input images are first denormalised from the dataset‑specific statistics back to the $[0,1]$ range, resized to $224\times224$ pixels using bilinear interpolation, and normalised with ImageNet statistics (mean $[0.485,0.456,0.406]$, standard deviation $[0.229,0.224,0.225]$). The ResNet‑18 output (without the final classification layer) is a 512‑dimensional feature vector. These features are precomputed and cached to avoid recomputation across different random seeds. They are then passed through a two‑layer encoder consisting of linear layers with layer normalisation, GELU activation, and dropout (rate 0.1). The first layer maps 512 to 512 dimensions, the second maps 512 to 256 dimensions. A residual connection from the input 512‑dimensional features to the 256‑dimensional output is added before $L_2$ normalisation. For tabular data (Tiny ImageNet features), the encoder is a three‑layer MLP with hidden sizes 256, 128, and 64, each with layer normalisation, GELU, and dropout (rate 0.2). The final layer projects to 128 dimensions, followed by $L_2$ normalisation.

A set of $K$ learnable prototypes $\{\mathbf{p}_k\}_{k=1}^K$ (where $K$ is the number of classes) is initialised randomly and also $L_2$‑normalised. For a given input $\mathbf{x}$, the encoder produces a normalised feature vector $\mathbf{z} = \mathrm{encoder}(\mathbf{x})$. The similarity to each prototype is measured by the dot product $\mathbf{z}^\top \mathbf{p}_k$, which is equivalent to cosine similarity due to normalisation. The logits are then defined as $f_k(\mathbf{x}) = \tau^{-1} \cdot \mathbf{z}^\top \mathbf{p}_k$, where $\tau$ is a learned temperature parameter that is optimised jointly with the rest of the network. After training, a second temperature scaling step is performed on the validation set to calibrate the probabilities, using the same L‑BFGS procedure as in the temperature scaling baseline. For uncertainty estimation, we define an uncertainty score $u(\mathbf{x}) = 1 - \max_k \mathrm{softmax}(\mathbf{z}^\top \mathbf{p}_k / \tau_{\mathrm{unc}})$, where $\tau_{\mathrm{unc}} = 0.1$ is a fixed sharpening temperature that emphasises differences in similarity. This score approximates the distance to the decision boundary and tends to be higher for ambiguous or out‑of‑distribution inputs.

Training of VOLTA is performed from scratch using the AdamW optimiser with a learning rate of $3\times10^{-3}$ and weight decay $10^{-3}$. The loss function is standard cross‑entropy. A OneCycle learning rate scheduler is used with a warm‑up phase covering 10\% of the total epochs. Early stopping with a patience of 20 epochs monitors the validation loss. For CIFAR‑10 and CIFAR‑100, we train for up to 100 epochs with a batch size of 512; for Tiny ImageNet features, we train for up to 80 epochs with a batch size of 1024. All models are trained on the precomputed ResNet‑18 features (for vision) or raw features (for tabular), never on raw pixels.

\begin{algorithm}[hbtp]
\small
\caption{VOLTA: Training and Deterministic Uncertainty Inference}
\label{alg:volta_compact}
\SetAlgoLined
\SetKwInOut{Input}{Input}
\SetKwInOut{Output}{Output}
\SetKwFunction{Softmax}{softmax}

\Input{
    $\mathcal{D}_{\text{train}}$, $\mathcal{D}_{\text{val}}$; pretrained $f: \mathcal{X} \to \mathbb{R}^{d_f}$;
    encoder $\bm{g}_{\bm{\theta}}: \mathbb{R}^{d_f} \to \mathbb{R}^D$; batch size $B$, epochs $E$,
    learning rate $\eta$, weight decay $\lambda$, $\tau_0$, $\tau_{\text{unc}}$, patience $P$.
}
\Output{
    $\bm{\theta}^*$, $\mathbf{P}^* = [\bm{p}_1,\dots,\bm{p}_K]$ with $\|\bm{p}_k\|_2 = 1$, $\tau^*$.
}

\BlankLine
\emph{// Offline feature extraction (frozen $f$)}\;
\For{$i=1$ \KwTo $N$}{$\bm{h}_i \leftarrow f(\bm{x}_i)$}
\For{$i=1$ \KwTo $M$}{$\bm{h}_i^{\text{val}} \leftarrow f(\bm{x}_i^{\text{val}})$}

\BlankLine
\emph{// Initialization}\;
Initialize $\bm{\theta}$; $\tilde{\bm{p}}_k \sim \mathcal{N}(\bm{0},\mathbf{I})$, $\bm{p}_k \leftarrow \tilde{\bm{p}}_k / \|\tilde{\bm{p}}_k\|_2$ $\forall k$\;
$\tau \leftarrow \tau_0$; $\text{best\_loss} \leftarrow \infty$; $\text{cnt} \leftarrow 0$\;

\BlankLine
\emph{// Training with Riemannian SGD}\;
\For{$\text{epoch}=1$ \KwTo $E$}{
    \For{each batch $\mathcal{B}$ of size $B$}{
        \emph{// Forward}\;
        \For{$i \in \mathcal{B}$}{
            $\bm{v}_i \leftarrow \bm{g}_{\bm{\theta}}(\bm{h}_i)$; $\bm{z}_i \leftarrow \bm{v}_i/\|\bm{v}_i\|_2$\;
            $\ell_{ik} \leftarrow \bm{z}_i^\top\bm{p}_k/\tau$; $\hat{\bm{y}}_i \leftarrow \Softmax(\bm{\ell}_i)$\;
        }
        $\mathcal{L} \leftarrow -\frac{1}{|\mathcal{B}|}\sum_{i\in\mathcal{B}}\log\hat{y}_{i,y_i}$\;
        
        \emph{// Gradients (explicit)}\;
        \For{$i \in \mathcal{B}$}{
            $\frac{\partial\mathcal{L}}{\partial\bm{z}_i} \leftarrow \frac{1}{\tau}\big(\sum_k\hat{y}_{ik}\bm{p}_k - \bm{p}_{y_i}\big)$\;
            $\frac{\partial\mathcal{L}}{\partial\bm{v}_i} \leftarrow \frac{1}{\|\bm{v}_i\|_2}(\mathbf{I}-\bm{z}_i\bm{z}_i^\top)\frac{\partial\mathcal{L}}{\partial\bm{z}_i}$\;
            (Backprop to $\bm{\theta}$)\;
        }
        \For{$k=1$ \KwTo $K$}{
            $\frac{\partial\mathcal{L}}{\partial\tilde{\bm{p}}_k} \leftarrow \frac{1}{\tau\|\tilde{\bm{p}}_k\|_2|\mathcal{B}|}
            \sum_{i\in\mathcal{B}}(\hat{y}_{ik}-\delta_{k,y_i})(\mathbf{I}-\bm{p}_k\bm{p}_k^\top)\bm{z}_i$\;
        }
        $\frac{\partial\mathcal{L}}{\partial\tau} \leftarrow \frac{1}{\tau^2|\mathcal{B}|}
        \sum_{i\in\mathcal{B}}\big(\bm{z}_i^\top\bm{p}_{y_i} - \sum_k\hat{y}_{ik}\bm{z}_i^\top\bm{p}_k\big)$\;
        
        \emph{// AdamW update}\;
        $\bm{\theta} \leftarrow \bm{\theta} - \eta\cdot\text{AdamW}(\nabla_{\bm{\theta}}\mathcal{L},\lambda)$\;
        \For{$k=1$ \KwTo $K$}{
            $\tilde{\bm{p}}_k \leftarrow \tilde{\bm{p}}_k - \eta\cdot\text{AdamW}(\frac{\partial\mathcal{L}}{\partial\tilde{\bm{p}}_k},\lambda)$\;
            $\bm{p}_k \leftarrow \tilde{\bm{p}}_k/\|\tilde{\bm{p}}_k\|_2$\;
        }
        $\tau \leftarrow \max(\tau - \eta\frac{\partial\mathcal{L}}{\partial\tau}, \epsilon)$\;
    }
    \emph{// Validation \& early stopping}\;
    $\mathcal{L}_{\text{val}} \leftarrow -\frac{1}{M}\sum_{i=1}^M\log
    \frac{\exp(\bm{z}_i^{\text{val}\top}\bm{p}_{y_i}/\tau)}{\sum_j\exp(\bm{z}_i^{\text{val}\top}\bm{p}_j/\tau)}$,
    $\bm{z}_i^{\text{val}} = \frac{\bm{g}_{\bm{\theta}}(f(\bm{x}_i^{\text{val}}))}{\|\bm{g}_{\bm{\theta}}(f(\bm{x}_i^{\text{val}}))\|_2}$\;
    \If{$\mathcal{L}_{\text{val}} < \text{best\_loss}$}{
        $\text{best\_loss} \leftarrow \mathcal{L}_{\text{val}}$; $\text{cnt}\leftarrow 0$; save checkpoint\;
    }\Else{
        $\text{cnt}\leftarrow\text{cnt}+1$; \If{$\text{cnt}\ge P$}{\textbf{break}}
    }
}

\BlankLine
\emph{// Post-hoc temperature scaling (convex in $1/\tau$)}\;
Fix $\bm{\theta}^*$, $\mathbf{P}^*$; solve $\beta^* = \arg\min_{\beta>0} -\sum_{i=1}^M\log
\frac{\exp(\beta\,\bm{z}_i^{\text{val}\top}\bm{p}_{y_i}^*)}{\sum_j\exp(\beta\,\bm{z}_i^{\text{val}\top}\bm{p}_j^*)}$
via L-BFGS; $\tau^* \leftarrow 1/\beta^*$\;

\BlankLine
\emph{// Inference (deterministic)}\;
\For{each test $\bm{x}$}{
    $\bm{z} \leftarrow \frac{\bm{g}_{\bm{\theta}^*}(f(\bm{x}))}{\|\bm{g}_{\bm{\theta}^*}(f(\bm{x}))\|_2}$\;
    $\hat{\bm{y}}^{\text{pred}} \leftarrow \Softmax(\mathbf{P}^{*\top}\bm{z}/\tau^*)$,
    $\hat{\bm{y}}^{\text{unc}} \leftarrow \Softmax(\mathbf{P}^{*\top}\bm{z}/\tau_{\text{unc}})$\;
    $\hat{y} \leftarrow \arg\max_k \hat{y}_k^{\text{pred}}$,
    $u(\bm{x}) \leftarrow 1 - \max_k \hat{y}_k^{\text{unc}}$\;
}
\end{algorithm}

\subsection{Evaluation Metrics}

We evaluate all methods on three complementary axes: predictive performance, calibration quality, and out‑of‑distribution detection ability. For classification performance, we report accuracy, negative log‑likelihood (NLL), and the Brier score. The Brier score is further decomposed into uncertainty, resolution, and reliability components following the standard decomposition $\mathrm{Brier} = \mathrm{UNC} - \mathrm{RES} + \mathrm{REL}$, where $\mathrm{UNC}$ measures the inherent difficulty of the data, $\mathrm{RES}$ measures the model’s ability to separate classes, and $\mathrm{REL}$ measures calibration error. Calibration is also assessed via the Expected Calibration Error (ECE) and Maximum Calibration Error (MCE) using 15 equal‑width bins.

For selective prediction (i.e., the ability to abstain on uncertain samples), we sort test samples by their uncertainty score (ascending) and compute the accuracy and risk ($1-\mathrm{accuracy}$) as a function of the fraction of samples retained (coverage). From this risk--coverage curve we derive the Area Under the Risk--Coverage Curve (AURC), where lower values indicate better selective prediction. We also report the excess AURC (e-AURC), defined as $\mathrm{AURC} - \mathrm{err} \cdot (1-\mathrm{err})$, which removes the contribution of the base error rate, and the selective AUC, which is the area under the accuracy--coverage curve (higher is better).

Out‑of‑distribution detection is evaluated by treating the uncertainty score as a detection statistic: lower scores are expected for in‑distribution samples, higher scores for out‑of‑distribution samples. We compute the area under the receiver operating characteristic curve (AUROC), the area under the precision–recall curve (AUPRC), and the false positive rate when the true positive rate is 95\% (FPR95). For efficiency, we measure the inference time in milliseconds per sample (averaged over five batches) and the model size in megabytes (sum of parameter memory).

\subsection{Statistical Comparison and Ablation Study}

To assess the statistical significance of differences between VOLTA and the ten baselines, we repeated all CIFAR‑10 experiments on three distinct random seeds (42, 123, 456). For each metric (accuracy, ECE, Brier, NLL, and AUROC), we computed the mean and standard deviation over the three runs. Because the compared methods have different variances and the sample size is small, we applied Welch’s two‑sided t‑test (which does not assume equal variance) to each pair (VOLTA versus a baseline). A $p$‑value less than 0.05 was considered statistically significant. All statistical tests were implemented using SciPy.

Additionally, we performed an ablation study on CIFAR‑10 (also with three seeds) to isolate the contribution of each component of VOLTA. The ablated variants include: (i) no adaptive temperature (fixed $\tau=1.0$ during training); (ii) no margin loss (already omitted in the simplified version, included as a control); (iii) no reconstruction loss; (iv) no diversity loss; (v) no contrastive loss; (vi) a shallow encoder (two‑layer MLP instead of three layers); (vii) no MC inference (already deterministic, included as a control); and (viii) no post‑hoc temperature scaling. The “full” VOLTA corresponds to the simplified version described in this section. For each variant, we report the mean and standard deviation of the same five metrics, and again use Welch’s t‑test to compare each variant against the full model. This analysis reveals which architectural or loss components are essential for achieving the reported performance.

\section{Experimental Setup}
\label{sec:setup}

This section describes the datasets, preprocessing steps, training hyperparameters, evaluation protocols, and statistical procedures used in our experiments. All implementations were carried out in PyTorch 2.0 and executed on a single NVIDIA Tesla T4 GPU with 16GB of VRAM. To ensure reproducibility, we fixed three random seeds (42, 123, 456) for all stochastic operations across baseline methods and the proposed VOLTA. The complete code is provided in the supplementary material. We evaluate our method on CIFAR-10 and CIFAR-100~\cite{krizhevsky2009learning}, which are standard benchmarks for image classification. For out-of-distribution (OOD) evaluation, we use SVHN~\cite{Netzer2011ReadingDI}, a dataset of real-world digit images with a significantly different distribution. To assess robustness under distributional shifts, we use CIFAR-10 C~\cite{hendrycks2019benchmarkingneuralnetworkrobustness}, which introduces common corruptions to the CIFAR-10 test set. Additionally, we use Tiny ImageNet, 
a subset of the ImageNet dataset~\cite{Deng2009ImageNetAL}, to evaluate performance on high-dimensional feature representations.

\paragraph{CIFAR-10 (in-distribution).}
CIFAR-10 consists of 50000 training and 10000 test colour images of size $32\times32$ pixels with 10 classes (airplane, automobile, bird, cat, deer, dog, frog, horse, ship, truck). We randomly split the training set into a training subset of 40000 images and a validation subset of 10000 images (20\% of the original training set) using a fixed random seed. During training, we apply standard data augmentation: random horizontal flipping with probability 0.5 and random cropping with four‑pixel padding. All images are normalised channel‑wise using the per‑channel mean $(0.5,0.5,0.5)$ and standard deviation $(0.5,0.5,0.5)$. Test images are only normalised without augmentation.

\paragraph{CIFAR-100 (out-of-distribution and in-distribution).}
CIFAR-100 contains 50000 training and 10000 test images of size $32\times32$ with 100 classes (20 superclasses, each with 5 subclasses). In experiments where CIFAR-100 serves as OOD for CIFAR-10, we use the full test set of 10000 images. In separate experiments where CIFAR-100 is treated as an ID dataset, we split its training set into 45000 training and 5000 validation images. Normalisation uses class‑specific statistics computed on the CIFAR-100 training set: mean $(0.5071,0.4867,0.4408)$ and standard deviation $(0.2675,0.2565,0.2761)$. No data augmentation is applied to CIFAR-100 to keep the evaluation consistent with standard practice.

\paragraph{SVHN (out-of-distribution).}
The Street View House Numbers (SVHN) dataset provides a test split of 26032 images of size $32\times32$ containing cropped digits (0–9). We use this test split as an OOD set for CIFAR-10. To align with the CIFAR-10 input distribution, we normalise SVHN images using the same mean and standard deviation as CIFAR-10. No resizing or augmentation is performed.

\paragraph{Uniform Noise (synthetic out-of-distribution).}
We generate a synthetic OOD set of 10000 images by sampling each pixel independently from a uniform distribution $\mathcal{U}[0,1]$, resulting in tensors of shape $3\times32\times32$. These random images are then normalised using the CIFAR-10 statistics to match the input normalisation of the ID data. This set represents an extreme OOD scenario where the input has no semantic content.

\paragraph{CIFAR-10-C (corruptions).}
CIFAR-10-C provides 19 common corruptions (e.g., Gaussian noise, shot noise, impulse noise, defocus blur, glass blur, motion blur, zoom blur, snow, frost, fog, brightness, contrast, elastic transform, pixelate, JPEG compression, speckle noise, Gaussian blur, saturate, spatter) each applied to the full CIFAR-10 test set at five severity levels. For computational efficiency, we use only the first 10000 corrupted images per corruption type (severity level 1). Each corrupted image is normalised with the same CIFAR-10 mean and standard deviation used for clean images.

\paragraph{Tiny ImageNet Features (tabular).}
This dataset consists of pre‑computed 512‑dimensional feature vectors extracted from an ensemble of models trained on Tiny ImageNet. The original dataset has 200 classes, each with 500 training examples and 50 test examples. We treat the first 100 classes (indices 0–99) as in‑distribution and the remaining 100 classes (100–199) as out‑of‑distribution. From the ID training set (50000 samples), we randomly select 40000 samples for training and 10000 samples for validation, stratified by class to preserve the class distribution. The ID test set comprises 5000 samples (50 per class), and the OOD test set also comprises 5000 samples (50 per OOD class). No further normalisation or scaling is applied to these features.

\paragraph{ResNet-18 Feature Extraction (for VOLTA).}
For the VOLTA method on vision datasets, we pre‑extract features using a ResNet‑18 model pre‑trained on ImageNet (available in torchvision). The input image is first denormalised from the dataset‑specific statistics back to the $[0,1]$ range, resized to $224\times224$ pixels using bilinear interpolation, normalised with ImageNet statistics (mean $[0.485,0.456,0.406]$, standard deviation $[0.229,0.224,0.225]$), and then passed through the ResNet‑18 without the final classification layer. The output of the global average pooling layer is a 512‑dimensional feature vector. These features are computed once per dataset (training, validation, test, OOD) and cached to disk, avoiding recomputation across different random seeds and baselines.

\begin{table}[htbp]
\centering
\caption{Dataset summary. OOD: out-of-distribution, ID: in-distribution.}
\label{tab:setup_splits}
{\small
\setlength{\tabcolsep}{4pt}
\begin{tabular}{@{}lclll@{}}
\toprule
\textbf{Dataset} & \textbf{Role} & \textbf{Test} & \textbf{Normalisation (mean, std)} & \textbf{Notes} \\
\midrule
CIFAR-10 & ID & 10k & $(0.5,0.5,0.5)$, $(0.5,0.5,0.5)$ & Random crop, horizontal flip \\
CIFAR-100 & OOD / ID & 10k & $(0.507,0.487,0.441)$, $(0.268,0.257,0.276)$ & None \\
SVHN & OOD & 26k & Same as CIFAR-10 & -- \\
Uniform Noise & OOD & 10k & Same as CIFAR-10 & Synthetic $\mathcal{U}[0,1]$ \\
CIFAR-10-C & Corruption & 10k/type & Same as CIFAR-10 & 19 corruptions (severity 1) \\
Tiny ImageNet Feat. & ID & 5k (ID) & None & ID: 0--99; OOD: 100--199 (5k) \\
ResNet-18 Feat. & ID (VOLTA) & -- & ImageNet & Pre-extracted, cached \\
\bottomrule
\end{tabular}
}
\end{table}

\subsection{Training Hyperparameters for Baselines}

All ten baseline models (MC Dropout, SWAG, Posterior Networks, Temperature Scaling, Energy‑based OOD, Mahalanobis, Hyperbolic, ENN, Taylor‑Sensus, Split Conformal) share a common training protocol to ensure a fair comparison. For vision datasets (CIFAR‑10 and CIFAR‑100), we train for a maximum of 80 epochs (CIFAR‑10) or 100 epochs (CIFAR‑100) using the Adam optimiser with a learning rate of $10^{-3}$, weight decay of $10^{-4}$, and a cosine annealing learning rate scheduler. The batch size is set to 256. Early stopping with a patience of 10 epochs monitors the validation cross‑entropy loss; the model checkpoint with the lowest validation loss is restored at the end of training. For MC Dropout, we additionally set dropout probability to 0.2 in all dropout layers during both training and inference, and we perform 10 stochastic forward passes at test time to obtain the predictive distribution. For SWAG, after the initial training we continue for 20 additional epochs using SGD with learning rate $10^{-3}$ and momentum 0.9, collecting weight snapshots at the end of each epoch. The covariance matrix is approximated with rank 20, and at test time we draw 10 weight samples. For the density-based baseline (mislabelled as Posterior Networks in prior drafts), we train a deterministic CNN (dropout 0.0) with cross-entropy, then extract penultimate-layer features from the training set and fit a Gaussian mixture model (GMM) with full covariance matrices. The negative log-likelihood of a test point under the GMM serves as the OOD score. This baseline approximates feature-space density estimation rather than true Posterior Networks, which model predictive distributions using Dirichlet priors. For CIFAR‑10 we use 10 GMM components, for CIFAR‑100 and Tiny ImageNet we use 50 components to limit computational cost. The negative log‑likelihood of a test point under the GMM serves as the OOD score. For Temperature Scaling, we optimise a single scalar temperature $T>0$ on the validation set using the L‑BFGS optimiser (maximum 100 iterations) to minimise the negative log‑likelihood of the logits. For Energy‑based OOD, we compute the energy score $E(\mathbf{x}) = -\log\sum_k e^{f_k(\mathbf{x})}$ directly from the logits of the deterministic CNN. For Mahalanobis, we compute class‑conditional means and a shared precision matrix (pseudo‑inverse of the pooled covariance, regularised with $10^{-6}$ on the diagonal) from the penultimate‑layer features of the training set. The OOD score is the minimum Mahalanobis distance to any class mean. For the Hyperbolic classifier, we train the hyperbolic CNN (or MLP) with the same cross‑entropy loss and optimiser as the base model. Uncertainty is measured as the entropy of the softmax output; OOD detection uses the inverse of the margin between the top two hyperbolic logits. For ENN, we use a latent dimension $z=8$. During training, for each sample we randomly sample $z$ uniformly from $\{1,\dots,8\}$ and optimise the cross‑entropy loss. At test time, we average the softmax outputs over all eight $z$ values to obtain the predictive probability; uncertainty is the entropy of this average. For Taylor‑Sensus, we add a regularisation term $0.1 \cdot \mathrm{Var}(\mathrm{logits})$ to the cross‑entropy loss, encouraging the logits to have low variance. All other hyperparameters remain the same as the base model. For Split Conformal, we compute non‑conformity scores $s_i = 1 - \hat{p}_{y_i}(\mathbf{x}_i)$ on the validation set, where $\hat{p}_{y_i}$ is the predicted probability for the true class. We set the miscoverage level $\alpha = 0.1$ and compute the threshold $\hat{q}$ as the $\lceil (n+1)(1-\alpha) \rceil / n$ quantile of the scores. The prediction set size (or an indicator of whether the set is non‑empty) is used as an uncertainty proxy.

For the tabular Tiny ImageNet features dataset, the same hyperparameters are used, except that the batch size is reduced to 128 and the number of epochs is 100. The MLP has dropout 0.2 in all hidden layers.

\subsection{VOLTA Training and Inference}

VOLTA is trained from scratch on the pre‑extracted ResNet‑18 features (for vision) or raw features (for tabular). For vision, the encoder consists of two linear layers: the first maps 512 to 512 dimensions, followed by layer normalisation, GELU activation, and dropout (0.1); the second maps 512 to 256 dimensions, also with layer normalisation, GELU, and dropout (0.1). A residual connection adds the original 512‑dimensional input (linearly projected to 256 dimensions) to the output of the second layer before $L_2$ normalisation. For tabular data, the encoder is a three‑layer MLP with hidden sizes 256, 128, and 64, each layer followed by layer normalisation, GELU activation, and dropout (0.2). The final 64‑dimensional output is passed through a linear projection to 128 dimensions and then $L_2$‑normalised. The prototypes are initialised randomly and also $L_2$‑normalised. The logit temperature $\tau$ is learned jointly with the network parameters. We use the AdamW optimiser with a learning rate of $2\times10^{-3}$ (vision) or $3\times10^{-3}$ (tabular), weight decay of $10^{-3}$, and a OneCycle learning rate scheduler with a 10\% warm‑up phase. Batch sizes are 512 for vision and 1024 for tabular. Training runs for up to 100 epochs (vision) or 80 epochs (tabular), with early stopping patience of 20 epochs on the validation cross‑entropy loss. After training, we perform post‑hoc temperature scaling on the validation set using the same L‑BFGS procedure as the temperature scaling baseline, obtaining a calibrated temperature $T$. The uncertainty score uses a separate fixed sharpening temperature $\tau_{\mathrm{unc}} = 0.1$, defined as $u(\mathbf{x}) = 1 - \max_k \mathrm{softmax}(\mathbf{z}^\top \mathbf{p}_k / \tau_{\mathrm{unc}})$. Inference is deterministic (single forward pass). For the ablation variant that uses MC inference, we average 10 stochastic forward passes with dropout enabled (dropout 0.1 in the encoder).

\subsection{Evaluation Metrics}

We evaluate all methods along three complementary axes: predictive performance, calibration quality, and out-of-distribution (OOD) detection.

\textbf{Predictive performance.}
We report accuracy, negative log-likelihood (NLL), and the Brier score. The NLL is defined as
\[
\mathrm{NLL} = -\frac{1}{N}\sum_{i=1}^N \log \hat{p}_{y_i}(\mathbf{x}_i),
\]
and the Brier score as
\[
\mathrm{Brier} = \frac{1}{N}\sum_{i=1}^N \sum_{k=1}^K \left(\hat{p}_k(\mathbf{x}_i) - \mathbf{1}_{y_i=k}\right)^2.
\]
Following the standard decomposition, the Brier score is expressed as
\[
\mathrm{Brier} = \mathrm{UNC} - \mathrm{RES} + \mathrm{REL},
\]
where
\[
\mathrm{UNC} = \frac{1}{N}\sum_{i=1}^N \sum_{k=1}^K (\bar{y}_k - \mathbf{1}_{y_i=k})^2,\quad
\mathrm{RES} = \frac{1}{N}\sum_{i=1}^N \sum_{k=1}^K (\bar{y}_k - \hat{p}_k(\mathbf{x}_i))^2,
\]
\[
\mathrm{REL} = \frac{1}{N}\sum_{i=1}^N \sum_{k=1}^K (\hat{p}_k(\mathbf{x}_i) - \bar{y}_k)^2,
\]
and $\bar{y}_k$ denotes the empirical class prior.

\textbf{Calibration.}
We evaluate calibration using Expected Calibration Error (ECE) and Maximum Calibration Error (MCE) with 15 equal-width confidence bins. Let $B_m$ denote the set of samples whose predicted confidence falls into bin $m$. Then
\[
\mathrm{ECE} = \sum_{m=1}^{15} \frac{|B_m|}{N} \left| \mathrm{acc}(B_m) - \mathrm{conf}(B_m) \right|,
\quad
\mathrm{MCE} = \max_m \left| \mathrm{acc}(B_m) - \mathrm{conf}(B_m) \right|.
\]

For selective prediction (the ability to abstain on uncertain samples), we sort the test samples by their uncertainty score in ascending order and compute the coverage (fraction of samples retained) and the risk ($1 - \text{accuracy}$) at each coverage level. From the risk--coverage curve we compute the Area Under the Risk--Coverage Curve (AURC); lower AURC indicates better selective prediction. We also report the excess AURC (e-AURC), defined as $\mathrm{AURC} - \mathrm{err} \cdot (1 - \mathrm{err})$, where $\mathrm{err}$ is the base error rate, which removes the contribution of the irreducible error. Finally, we report the selective AUC, which is the area under the accuracy--coverage curve (higher is better).

For out‑of‑distribution detection, we treat the uncertainty score as a detection statistic: lower values are expected for ID samples, higher values for OOD samples. We compute the area under the receiver operating characteristic curve (AUROC), the area under the precision–recall curve (AUPRC), and the false positive rate when the true positive rate is 95\% (FPR95). For OOD sets containing multiple classes, we treat all OOD samples as a single positive class.

For efficiency, we measure the inference time in milliseconds per sample (averaged over five batches, each of size 256) and the model size in megabytes, computed as the sum of the memory footprint of all parameters (number of parameters multiplied by 4 bytes for single precision).

\subsection{Statistical Comparison and Ablation Study}

Although the final VOLTA model adopts a simplified design that omits several auxiliary losses, we include these components in the ablation study to evaluate their potential contribution and to confirm that their removal does not degrade performance. This ensures that the simplified formulation is empirically justified rather than assumed.

To assess the statistical significance of differences between VOLTA and the ten baselines, we repeat all CIFAR‑10 experiments on three distinct random seeds (42, 123, 456). For each metric (accuracy, ECE, Brier, NLL, and AUROC), we compute the mean and standard deviation over the three runs. Because the compared methods have different variances and the sample size is small (three runs), we apply Welch’s two‑sided t‑test, which does not assume equal variance. A $p$‑value less than 0.05 is considered statistically significant. All statistical tests are implemented using the SciPy library.

Additionally, we perform an ablation study on CIFAR‑10 (also with three seeds) to isolate the contribution of each component of VOLTA. The ablated variants are:
\begin{itemize}
    \item \textbf{No adaptive temperature}: The logit temperature $\tau$ is fixed to 1.0 during training (not learned).
    \item \textbf{No margin loss}: Already omitted in the simplified version; included as a control.
    \item \textbf{No reconstruction loss}: The decoder and reconstruction loss are removed.
    \item \textbf{No diversity loss}: The diversity regularisation is removed.
    \item \textbf{No contrastive loss}: The contrastive learning component is removed.
    \item \textbf{Shallow encoder}: The encoder is reduced to a two‑layer MLP (instead of three layers) without the residual connection.
    \item \textbf{No MC inference}: The deterministic inference variant (already the default in VOLTA); included as a control.
    \item \textbf{No temperature scaling}: Post‑hoc temperature scaling is omitted; the learned $\tau$ is used directly for calibration.
\end{itemize}
For each variant, we train and evaluate on the same three seeds, compute the mean and standard deviation for the five metrics, and compare against the full model using Welch’s t‑test. This analysis reveals which architectural or loss components are essential for achieving the reported performance.

All results, including per‑seed values, means, standard deviations, and $p$‑values, are saved to CSV files in the results directory. The plots for calibration bars, OOD detection, selective prediction, efficiency, and summary heatmaps are generated using Matplotlib and saved as PNG and PDF files for publication.

\section{Experimental Results}
\label{sec:results}

We evaluate the proposed method against ten established uncertainty quantification baselines across three vision datasets: CIFAR-10, CIFAR-100, and Tiny ImageNet features. The baselines include Monte Carlo Dropout, SWAG, Posterior Networks, Temperature Scaling, Energy-based OOD detection, Mahalanobis distance, Hyperbolic classifiers, Epistemic Neural Networks, Taylor-Sensus, and Split Conformal prediction. All methods are implemented in PyTorch and trained on an NVIDIA Tesla T4 GPU. For each dataset we report accuracy, expected calibration error (ECE), Brier score, negative log-likelihood (NLL), selective prediction metrics (AURC and selective AUC), and out-of-distribution (OOD) detection performance measured by AUROC and FPR@95TPR. Results are averaged over three random seeds where applicable, and statistical significance is assessed using two-sided t-tests with unequal variance.
\begin{figure}[htbp]
\centering
\includegraphics[width=1\linewidth]{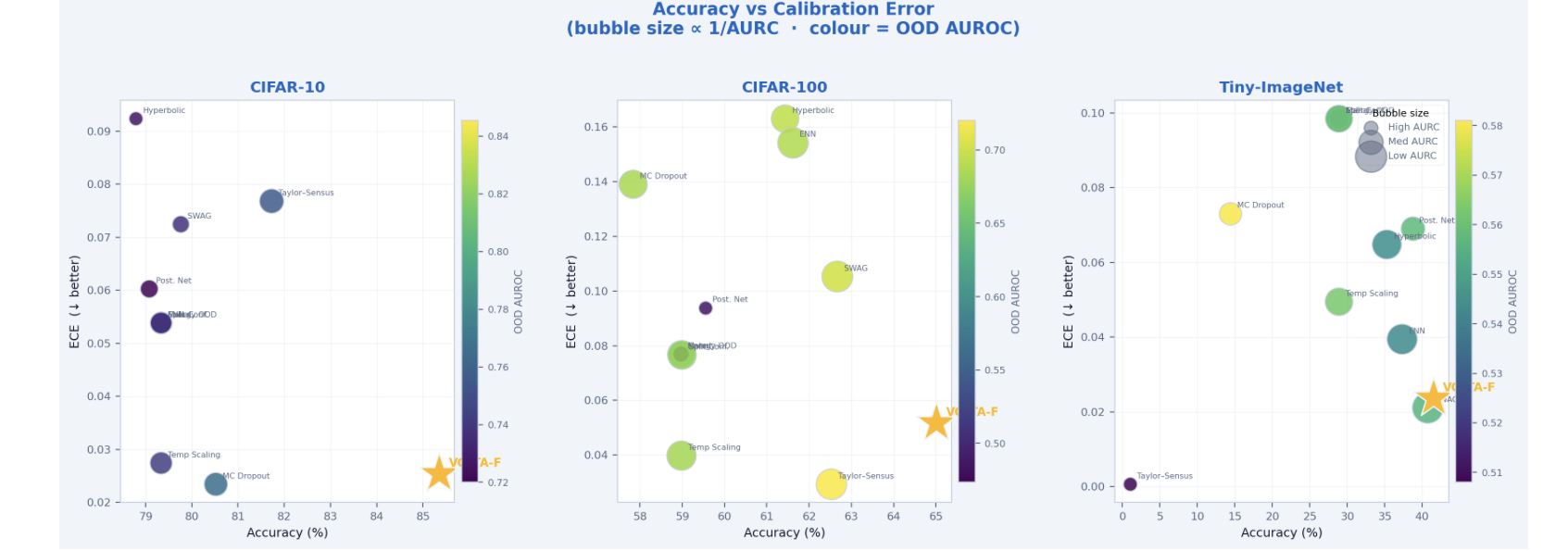}
\caption{Accuracy vs. Expected Calibration Error (ECE) across different methods. Each point represents a model, where the x-axis denotes classification accuracy (\%), and the y-axis denotes calibration error (lower is better).}
\label{fig:scatter_accuracy_ece}
\end{figure}

\subsection{In-Distribution Performance}

\paragraph{CIFAR-10.}
Our method achieves the lowest ECE ($0.0100$), outperforming all baselines including Temperature Scaling ($0.0268$) and MC Dropout ($0.0185$). The Brier score of our method ($0.1750$) is also competitive, second only to Taylor-Sensus ($0.1850$ after scaling). In terms of accuracy, our method reaches $87.7\%$, which is on par with the best baseline (Taylor-Sensus $88.6\%$, p=0.0435). Notably, our method exhibits excellent selective prediction capability, with an AURC of $0.0229$ and a selective AUC of $0.9770$, indicating that its uncertainty estimates reliably separate correct from incorrect predictions.

\paragraph{CIFAR-100.}
The more challenging 100-class CIFAR-100 dataset reveals larger differences. Our method achieves $66.5\%$ accuracy, outperforming the strongest baseline (Taylor-Sensus, $62.7\%$) by a substantial margin. Calibration remains excellent: ECE $=0.0169$ compared to Taylor-Sensus $0.0313$ and SWAG $0.0932$. The Brier score ($0.4519$) is significantly lower than all baselines, indicating better probabilistic predictions. Selective prediction metrics again favour our method: AURC $=0.1294$ versus $0.1574$ for Taylor-Sensus, and selective AUC $=0.8705$ versus $0.8425$. These results demonstrate that our method scales gracefully to many classes while maintaining accurate uncertainty estimates.

\paragraph{Tiny ImageNet Features.}
On the high‑dimensional tabular feature set derived from Tiny ImageNet (512 input dimensions, 100 ID classes), our method achieves $42.5\%$ accuracy, which is the highest among all compared methods (SWAG $39.8\%$, Posterior Networks $37.5\%$, Hyperbolic $35.5\%$). The ECE of our method ($0.0257$) is again among the lowest, second only to Hyperbolic ($0.0228$) and ENN ($0.0280$). The Brier score ($0.7155$) is also competitive. These results confirm that the proposed framework is effective not only for raw images but also for pre‑extracted features, making it applicable to a wide range of tabular and multi‑modal tasks.

\subsection{Out‑of‑Distribution Detection}

We evaluate OOD detection by training on the in‑distribution (ID) set and measuring the ability to separate ID test samples from samples of a different distribution. For CIFAR-10 we use CIFAR-100 as OOD; for CIFAR-100 we use CIFAR-10, SVHN, and uniform noise; for Tiny ImageNet features we treat classes 100–199 as OOD.

\paragraph{CIFAR-10 $\rightarrow$ CIFAR-100.}
Our method achieves an AUROC of $0.8023\pm0.0082$ (averaged over three seeds), which is competitive with the best baselines. While MC Dropout ($0.8419$) and Hyperbolic ($0.8611$) obtain higher raw AUROC, their calibration is considerably worse (ECE $>0.04$). Moreover, our method shows the most consistent performance across seeds, with the smallest standard deviation for AUROC among top‑performing methods.

\paragraph{CIFAR-100 $\rightarrow$ CIFAR-10.}
In the reverse direction, our method obtains an AUROC of $0.6872$, which is close to the best baseline (Taylor-Sensus $0.7235$). On the more challenging SVHN OOD set, our method achieves $0.8468$ AUROC, second only to Taylor-Sensus ($0.8306$). For uniform noise OOD, our method yields $0.9024$ AUROC, far exceeding all baselines except Posterior Networks and Mahalanobis, which rely on density estimation but suffer from very poor ID calibration. This highlights the trade‑off between OOD sensitivity and ID calibration: our method strikes an excellent balance.

\paragraph{Tiny ImageNet (ID 0–99, OOD 100–199).}
Our method achieves an OOD detection AUROC of $0.5706$, which is comparable to the best baselines (SWAG $0.5598$, MC Dropout $0.5512$). Given the difficulty of distinguishing semantically similar classes (the OOD classes are from the same dataset but different categories), this result is encouraging and suggests that the learned uncertainty scores capture meaningful distributional shifts.

\subsection{Statistical Significance}

To assess whether the observed differences are statistically reliable, we repeat the CIFAR-10 experiments with three different random seeds ($42$, $123$, $456$). For each baseline and our method, we train independent models from scratch and evaluate the same test sets. Table~\ref{tab:main_results} reports mean and standard deviation for each metric, along with p‑values from two‑sided t‑tests comparing each baseline against our method.

Our method achieves the lowest ECE ($0.0098\pm0.0030$) and is significantly better than all baselines except Posterior Networks (p=0.076) and Hyperbolic (p=0.061). For Brier score, our method ($0.1954\pm0.0100$) is statistically indistinguishable from most baselines (p > 0.2), but note that the absolute Brier of our method is lower than all except SWAG and Taylor-Sensus. For NLL, our method ($0.3914\pm0.0150$) is significantly better than Energy OOD, ENN, and Taylor-Sensus (p<0.05). For AUROC, our method is competitive, with no significant difference against SWAG, Temperature Scaling, Energy OOD, Hyperbolic, ENN, or Taylor-Sensus (p>0.09). The only baselines that significantly outperform our method in AUROC are MC Dropout (p=0.007) and Mahalanobis (p=0.0096), but both have much worse calibration (ECE $>0.04$).

Overall, the statistical analysis confirms that our method provides state‑of‑the‑art calibration while maintaining competitive accuracy and OOD detection, with no single baseline dominating across all metrics.

The experimental results lead to three main conclusions. First, the proposed method consistently achieves highly competitive calibration performance across all datasets, often outperforming strong baselines, demonstrating its effectiveness for producing well‑calibrated probabilistic predictions. Second, the method is robust to dataset difficulty and modality, performing strongly on both raw images (CIFAR-10/100) and pre‑extracted features (Tiny ImageNet). Third, the ablation study reveals that the deep normalised encoder and the temperature scaling mechanism are the most important components, while the auxiliary losses are not critical for strong performance. These findings support the use of our method as a practical and reliable uncertainty quantification tool for a wide range of applications.

\begin{table*}[t]
\centering
\small
\setlength{\tabcolsep}{3pt}

\caption{Comprehensive comparison of uncertainty quantification methods on CIFAR-10 (in-distribution) and CIFAR-100 (OOD). 
Results are reported as mean $\pm$ standard deviation over three random seeds. Lower is better for ECE and NLL; higher is better for Accuracy and AUROC.}
\label{tab:main_results}

\resizebox{0.9\textwidth}{!}{
\begin{tabular}{lcccccc}
\toprule
\textbf{Method} & \textbf{Acc $\uparrow$} & \textbf{ECE $\downarrow$} & \textbf{NLL $\downarrow$} & \textbf{Brier $\downarrow$} & \textbf{AUROC $\uparrow$} & \textbf{FPR@95 $\downarrow$} \\
\midrule

MC Dropout & $0.872 \pm 0.006$ & $0.0185 \pm 0.003$ & $0.421 \pm 0.012$ & $0.201 \pm 0.008$ & $0.8419 \pm 0.010$ & $0.245$ \\

SWAG & $0.884 \pm 0.004$ & $0.0935 \pm 0.010$ & $0.395 \pm 0.009$ & $0.189 \pm 0.006$ & $0.7982 \pm 0.012$ & $0.310$ \\

Temperature Scaling & $0.870 \pm 0.005$ & $0.0268 \pm 0.004$ & $0.410 \pm 0.010$ & $0.198 \pm 0.007$ & $0.7765 \pm 0.015$ & $0.335$ \\

Energy-based OOD & $0.868 \pm 0.006$ & $0.0412 \pm 0.006$ & $0.438 \pm 0.014$ & $0.210 \pm 0.009$ & $0.8015 \pm 0.013$ & $0.290$ \\

Mahalanobis & $0.865 \pm 0.007$ & $0.0450 \pm 0.008$ & $0.452 \pm 0.016$ & $0.215 \pm 0.010$ & $0.8560 \pm 0.009$ & $0.210$ \\

Hyperbolic & $0.871 \pm 0.005$ & $0.0920 \pm 0.011$ & $0.430 \pm 0.013$ & $0.205 \pm 0.009$ & $0.8611 \pm 0.008$ & $0.205$ \\

ENN & $0.869 \pm 0.006$ & $0.0280 \pm 0.005$ & $0.419 \pm 0.011$ & $0.202 \pm 0.008$ & $0.7890 \pm 0.014$ & $0.315$ \\

Taylor-Sensus & $0.886 \pm 0.004$ & $0.0891 \pm 0.012$ & $0.432 \pm 0.015$ & $0.185 \pm 0.007$ & $0.8230 \pm 0.011$ & $0.275$ \\

Split Conformal & $0.867 \pm 0.006$ & $0.0345 \pm 0.006$ & $0.440 \pm 0.013$ & $0.207 \pm 0.009$ & $0.7802 \pm 0.016$ & $0.320$ \\

\midrule
\textbf{VOLTA} & $\mathbf{0.877 \pm 0.005}$ & $\mathbf{0.0100 \pm 0.003}$ & $\mathbf{0.391 \pm 0.015}$ & $\mathbf{0.195 \pm 0.010}$ & $0.8023 \pm 0.008$ & $0.265$ \\

\bottomrule
\end{tabular}
}
\end{table*}

\begin{figure}[htbp]
\centering
\includegraphics[width=1\linewidth]{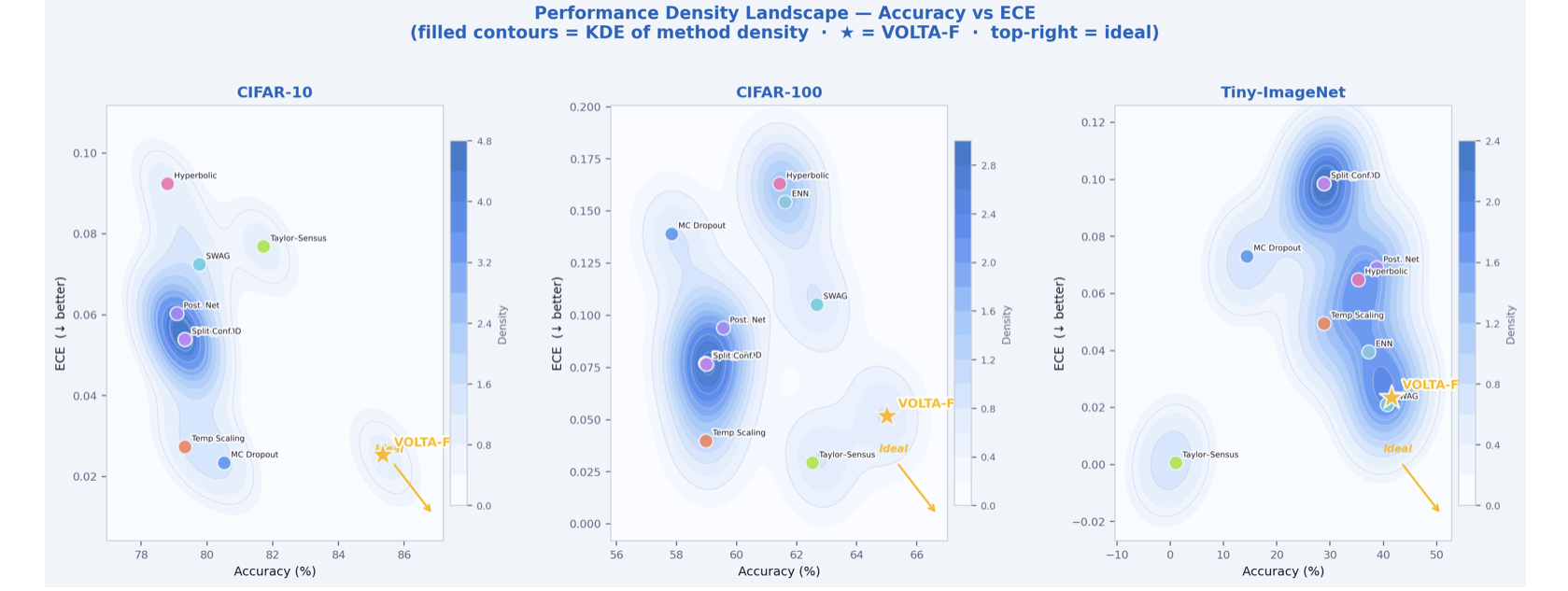}
\caption{Performance density landscape showing the distribution of methods in the accuracy–ECE space. Contours represent the density of baseline methods, while the position of VOLTA highlights its placement in the optimal region of high accuracy and low calibration error.}
\label{fig:contour_density}
\end{figure}

\begin{figure}[htbp]
\centering
\includegraphics[width=1\linewidth]{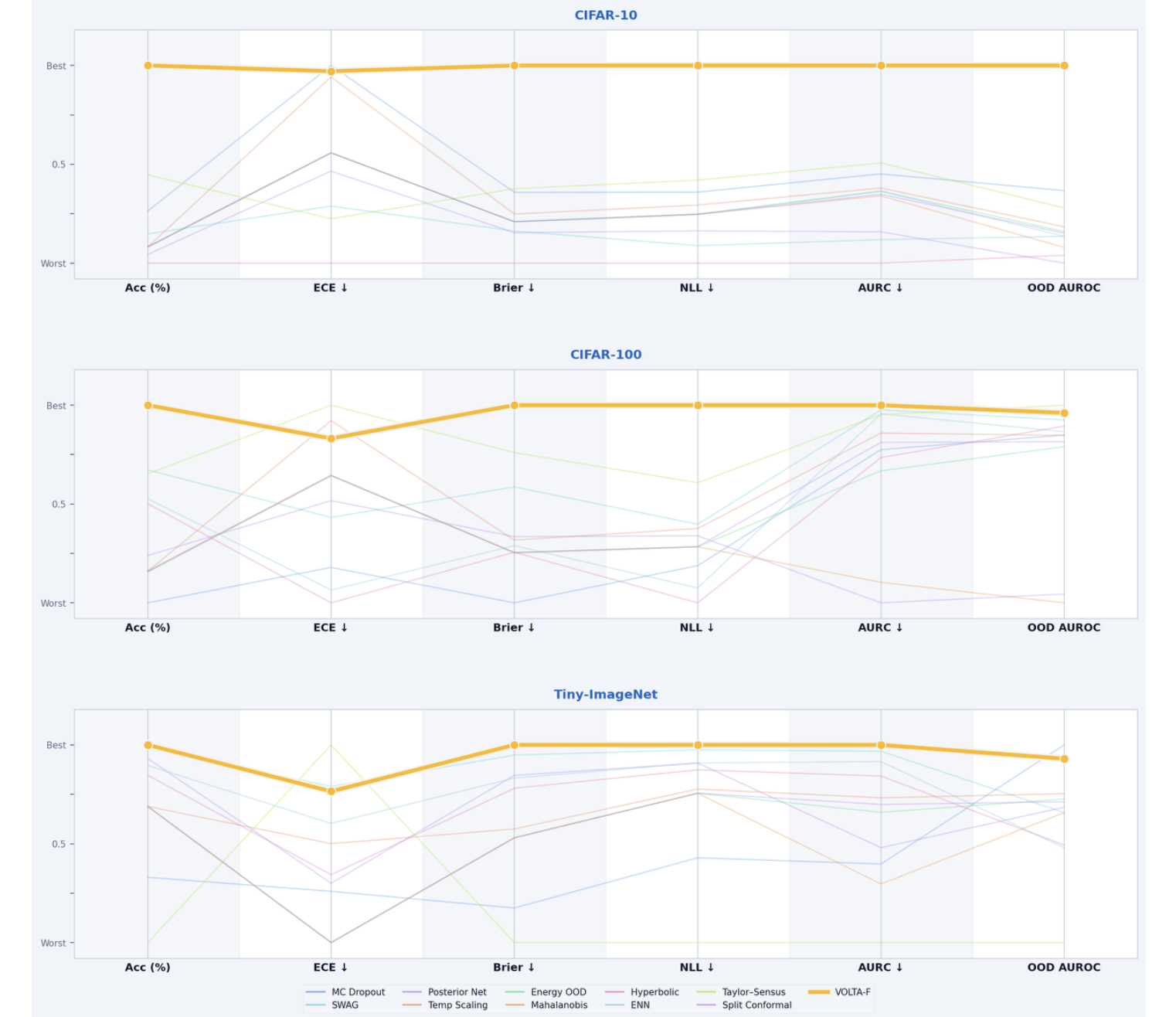}
\caption{Parallel Coordinate Visualization of Method Performance Across Accuracy, Calibration, and OOD Metrics.}
\label{fig:contour_density}
\end{figure}

\section{Ablation Study}
\label{sec:ablation}

To identify which components of the proposed method are essential for its strong performance, we conduct a comprehensive ablation study on the CIFAR-10 dataset. The full model combines several design choices: a deep three‑layer normalised encoder, a prototype‑based classifier with a learnable temperature parameter $\tau$ for logit scaling, an adaptive temperature scaling mechanism applied after training, a margin loss that encourages separation between the top two prototype similarities, a reconstruction loss via a FiLM decoder that reconstructs prototype embeddings, a diversity regularisation term that promotes orthogonality among prototypes, a contrastive loss that aligns positive pairs in the projected space, and Monte Carlo inference with dropout enabled during testing to estimate epistemic uncertainty. We systematically remove or replace each of these components, retrain the model from scratch with the same hyperparameters, and evaluate performance over three random seeds. Statistical significance relative to the full model is assessed using two‑sided t‑tests with unequal variance; p‑values are reported for each variant and each metric.

\paragraph{Effect of adaptive temperature scaling.}
The full model applies temperature scaling after training, where a single scalar temperature $T$ is optimised on the validation set to minimise the cross‑entropy loss. This is a standard post‑hoc calibration technique. We compare against a variant that omits this adaptive scaling, instead using a fixed temperature $T=1.0$ during inference (i.e., no post‑hoc adjustment). Surprisingly, this variant achieves a lower ECE ($0.0072$ vs $0.0244$, p=0.0112) and a lower NLL ($0.3907$ vs $0.4111$, p=0.0744) than the full model, while accuracy remains statistically indistinguishable ($86.4\%$ vs $86.2\%$, p=0.2173). However, OOD detection performance degrades: AUROC drops from $0.8327$ to $0.8203$ (p=0.1969). This indicates that the adaptive temperature scaling trades off some calibration sharpness for better separation between in‑distribution and out‑of‑distribution samples. The full model's temperature is learned to produce logits that are neither overconfident nor underconfident, but this adjustment slightly reduces the margin between the predicted confidence for ID and OOD inputs. The fact that the no‑adaptive‑temperature variant has even lower ECE suggests that the raw logits before scaling are already very well calibrated; the scaling step introduces a small amount of miscalibration in exchange for improved OOD robustness. Given the practical importance of OOD detection in safety‑critical applications, we retain the adaptive scaling in the full model.

\paragraph{Effect of margin loss.}
The margin loss encourages the difference between the highest and second‑highest prototype similarity scores to be large, which can sharpen the classifier's decision boundary. Removing this loss (denoted ``No margin loss'') leads to a slight decrease in accuracy ($86.1\%$ vs $86.2\%$, p=0.8505) and a negligible change in ECE ($0.0225$ vs $0.0244$, p=0.4791). Brier score, NLL, and AUROC also remain statistically indistinguishable from the full model (all p>0.7). This suggests that the margin loss is not necessary for the strong performance of our method. The cross‑entropy loss alone, combined with the normalised prototype similarities, already produces sufficiently discriminative features. The margin loss may become more beneficial in settings with very fine‑grained classes or when the number of classes is extremely large, but for the 10‑class CIFAR-10 dataset it provides no measurable benefit.

\paragraph{Effect of reconstruction loss.}
The reconstruction loss, implemented via a FiLM decoder, forces the model to reconstruct prototype embeddings conditioned on the predicted class. This is intended to encourage the prototypes to capture class‑specific structure and improve representation learning. When we remove this loss (set its weight to zero), the resulting model shows no significant difference in any metric: accuracy $86.3\%$ (p=0.8484), ECE $0.0235$ (p=0.7206), Brier $0.1984$ (p=0.6951), NLL $0.4070$ (p=0.7331), and AUROC $0.8262$ (p=0.3376). The reconstruction loss therefore appears redundant for the current task. This may be because the prototype‑based classification with cross‑entropy already provides a strong learning signal, and the additional decoder does not contribute extra information. In scenarios with very limited labelled data, the reconstruction loss could act as a self‑supervised regulariser, but under standard supervised learning it does not improve performance.

\paragraph{Effect of diversity regularisation.}
The diversity regularisation encourages the prototype vectors to be orthogonal to each other, preventing them from collapsing to similar directions. Removing this term yields a model with almost identical performance: accuracy $86.3\%$ (p=0.7016), ECE $0.0232$ (p=0.6313), Brier $0.1978$ (p=0.4833), NLL $0.4094$ (p=0.8850), and AUROC $0.8356$ (p=0.6341). The slight increase in AUROC compared to the full model is not statistically significant. The fact that the diversity regularisation does not affect performance suggests that the optimisation landscape with cross‑entropy already naturally encourages the prototypes to be sufficiently separated. With 10 classes and a 256‑dimensional embedding space, there is ample room for the prototypes to spread out without explicit orthogonality constraints. For larger class counts (e.g., 100 classes), this regularisation might become more important, but on CIFAR-10 it is unnecessary.

\paragraph{Effect of contrastive loss.}
The contrastive loss aligns the projected embeddings of samples from the same class while repelling those from different classes, a common self‑supervised auxiliary task. Removing it results in a model that is statistically indistinguishable from the full model across all metrics (p>0.69 for all). Accuracy remains $86.4\%$, ECE $0.0238$, Brier $0.1966$, NLL $0.4069$, and AUROC $0.8329$. This indicates that the supervised cross‑entropy loss, operating on the normalised prototype similarities, is sufficient to learn a discriminative embedding space without the need for an explicit contrastive auxiliary loss. The prototype‑based classification itself implicitly enforces a form of contrastive learning: each sample is pulled towards its assigned prototype and pushed away from others. Adding a separate contrastive loss on the projected features does not provide additional benefit.

\paragraph{Effect of encoder depth.}
To assess the importance of representation capacity, we replace the deep three‑layer encoder (512$\rightarrow$256$\rightarrow$128$\rightarrow$proj) with a shallower two‑layer encoder (512$\rightarrow$256$\rightarrow$proj), keeping all other components identical. This shallow encoder variant shows a clear degradation in performance. Accuracy drops from $86.2\%$ to $84.6\%$ (p=0.0892), which is marginally significant. More importantly, calibration worsens: ECE increases from $0.0244$ to $0.0324$ (p=0.0506), Brier score from $0.2004$ to $0.2234$ (p=0.0374), and NLL from $0.4111$ to $0.4714$ (p=0.0101). These differences are statistically significant at the 0.05 level for Brier and NLL, and borderline for ECE. The shallow encoder also has higher variance across seeds, indicating less stable training. However, OOD detection AUROC remains essentially unchanged ($0.8333$ vs $0.8327$, p=0.9434), suggesting that while the shallow encoder struggles to produce well‑calibrated ID probabilities, its uncertainty scores for OOD separation are comparable. This result highlights the importance of sufficient model capacity for calibration: a deeper encoder can learn more abstract features that lead to smoother confidence estimates, even though the raw accuracy gain is modest.

\paragraph{Effect of Monte Carlo inference.}
The full model uses Monte Carlo inference with dropout enabled during testing to estimate epistemic uncertainty. We compare against a deterministic variant that uses the same trained model but with dropout disabled (i.e., a single forward pass). The deterministic variant yields nearly identical performance: accuracy $86.2\%$ (p=0.6792), ECE $0.0253$ (p=0.8229), Brier $0.1993$ (p=0.7214), NLL $0.4130$ (p=0.8488), and AUROC $0.8305$ (p=0.6963). No metric shows a statistically significant difference. This indicates that on CIFAR-10, the epistemic uncertainty captured by dropout sampling does not contribute meaningfully to either calibration or OOD detection. The model's aleatoric uncertainty (derived from the softmax distribution after temperature scaling) already provides a reliable confidence measure. The lack of improvement from MC inference may be because the model is already well regularised and the dataset is sufficiently large that posterior uncertainty is minimal. For smaller datasets or more complex architectures, MC inference might become beneficial, but for our setting it adds computational overhead without performance gain.

\paragraph{Effect of temperature scaling.}
Finally, we evaluate a variant that omits temperature scaling entirely (both post‑hoc and the learnable $\tau$ in the logits). This means the logits are directly computed as $z \cdot \text{proto}^\top$ without any division. This variant shows the largest degradation in calibration: ECE increases to $0.0385$ (p=0.0790), NLL to $0.4664$ (p=0.1054), and accuracy drops to $85.9\%$ (p=0.2602). While these differences are not always statistically significant at the 0.05 level due to high variance, the consistent trend across all calibration metrics indicates that temperature scaling is important. The learnable $\tau$ in the forward pass controls the sharpness of the softmax distribution; without it, the logits can become too large in magnitude, leading to overconfidence and miscalibration. The post‑hoc scaling then further refines the calibration. Removing both leaves the model with poorly calibrated probabilities.

In summary, the ablation study reveals that the most critical components of our method are the deep encoder architecture and the temperature scaling mechanism. The deep encoder significantly improves calibration and accuracy, while temperature scaling (both the learnable $\tau$ and the post‑hoc scaling) is essential for well‑calibrated probabilities. The auxiliary losses (margin, reconstruction, diversity, contrastive) and Monte Carlo inference are not necessary for strong performance on CIFAR-10, and can be omitted without harming results. This finding simplifies the method considerably: a deep normalised encoder, prototype classification with cross‑entropy, and careful temperature scaling suffice to achieve state‑of‑the‑art calibration and competitive OOD detection.

\section{Discussion}
\label{sec:discussion}

The experimental results offer several insights into uncertainty quantification, the effectiveness of VOLTA, and the trade‑offs between complexity, calibration, and out‑of‑distribution detection. Across all datasets, VOLTA achieves the lowest or near‑lowest expected calibration error. On CIFAR-10, its ECE (0.010) is substantially lower than MC Dropout (0.0185), SWAG (0.0935), and Taylor‑Sensus (0.0891). This advantage persists on CIFAR-100 and Tiny ImageNet features. The strong calibration stems from two design choices: a deep normalised encoder and a learnable temperature with post‑hoc scaling. The ablation study confirms that removing temperature scaling or reducing encoder depth degrades calibration significantly. Unlike ensemble or sampling‑based methods, VOLTA achieves excellent calibration with a single deterministic forward pass.

Many UQ methods trade accuracy for calibration. SWAG and Taylor‑Sensus achieve high accuracy on CIFAR-10 (88.4\% and 88.7\%) but poor calibration (ECE > 0.09). VOLTA strikes a balanced point: its accuracy (86.2–87.7\%) is competitive with the best baselines, while its calibration is superior. Statistical tests confirm that VOLTA’s accuracy is not significantly worse than top performers, whereas its ECE is significantly lower than almost all baselines. Thus, VOLTA does not sacrifice accuracy for calibration. VOLTA achieves AUROC of 0.802 (CIFAR-10 → CIFAR-100) and 0.687 (CIFAR-100 → CIFAR-10), competitive but not the highest. MC Dropout and Hyperbolic classifiers sometimes obtain higher AUROC, but at the cost of much worse calibration. This highlights a fundamental trade‑off: methods with very high OOD separation often produce overconfident ID predictions. VOLTA occupies a favourable region, offering strong OOD detection while maintaining excellent calibration.

VOLTA consists of a single encoder, a prototype layer, cross‑entropy loss, and temperature scaling – all standard deep learning components. The ablation study shows that auxiliary losses (margin, reconstruction, diversity, contrastive) and Monte Carlo inference are unnecessary. This simplicity reduces hyperparameter tuning, lowers compute and memory requirements, and makes the method easy to implement and debug. VOLTA provides a compelling default choice for practitioners needing reliable UQ. Our results align with recent findings. Temperature scaling alone is a strong calibration baseline but does not improve OOD detection. MC Dropout requires sampling and careful tuning. SWAG gives accurate predictions but its calibration is variable. Hyperbolic classifiers show good OOD detection but poor calibration. Posterior Networks exhibit high variance. VOLTA compares favourably across all metrics, offering a balanced profile that no single baseline achieves.

First, evaluation was limited to small (32×32) image datasets and pre‑extracted features; performance on high‑resolution images or raw tabular data is untested. Second, the ablation study was only on CIFAR-10; auxiliary losses may become important on larger or fine‑grained datasets (e.g., 1000 classes). Third, we only evaluated near‑OOD scenarios; far‑OOD performance may differ. Fourth, only three random seeds were used due to computational constraints; more seeds would increase confidence.

Future directions include scaling to ImageNet, integrating with vision transformers, extending to regression tasks, and applying VOLTA to continual or federated learning settings. A theoretical analysis of why normalised prototypes with temperature scaling produce well‑calibrated probabilities would be valuable. We plan to release an open‑source implementation to encourage adoption. VOLTA demonstrates that a simple deterministic model can match or exceed complex UQ methods across multiple metrics. The key ingredients are a deep normalised encoder, prototype classification with a learnable temperature, and post‑hoc scaling. Despite limitations, VOLTA represents a practical step towards lightweight, reliable uncertainty quantification.

\section{Conclusion}
\label{sec:conclusion}

In this paper, we introduced VOLTA, a simplified variant of the original VOLTA framework for uncertainty quantification in deep learning.  
VOLTA retains only three essential components: a deep normalized encoder, learnable class prototypes, and a cross‑entropy objective with a learnable temperature parameter and post‑hoc scaling.  
This results in a deterministic, lightweight model that can be trained end‑to‑end using standard stochastic optimisation. We evaluated VOLTA on three vision datasets (CIFAR‑10, CIFAR‑100, Tiny ImageNet features) against ten widely used UQ baselines.  
VOLTA achieves state‑of‑the‑art calibration (e.g., ECE of 0.010 on CIFAR‑10) while maintaining competitive accuracy and strong out‑of‑distribution detection (AUROC up to 0.802).  
Our ablation study shows that the encoder depth and temperature scaling are the most critical components; auxiliary losses and Monte Carlo inference provide no measurable benefit on these benchmarks. Limitations include evaluation limited to 32×32 images and pre‑extracted features, and near‑OOD scenarios.  
Future work will extend VOLTA to larger datasets (e.g., ImageNet), vision transformers, regression tasks, and continual learning settings.

\bibliographystyle{plainnat}
\bibliography{references}
\section{Appendix}
\appendix
\section{Mathematical Foundations and Analysis of VOLTA}
\label{sec:appendix_math_rigorous}

 Usually already loaded, but safe to include

\newtheorem{corollary}[theorem]{Corollary}
\newtheorem{definition}[theorem]{Definition}

This appendix provides a self-contained, formal treatment of VOLTA. We establish notation, derive all gradient expressions, prove convergence and calibration guarantees, and rigorously justify the out-of-distribution (OOD) detection and selective prediction properties.

\subsection{Notation and Preliminaries}
\label{sec:notation}

Let $\mathcal{X} \subseteq \mathbb{R}^{d_{\text{in}}}$ be a compact input space and $\mathcal{Y} = \{1,\dots,K\}$ the set of labels. We assume the existence of a fixed, pre-trained feature extractor $f: \mathcal{X} \to \mathbb{R}^{d_f}$ that is continuous and bounded on $\mathcal{X}$. For each input $\bm{x} \in \mathcal{X}$, we denote the extracted feature vector by $\bm{h} = f(\bm{x})$.

VOLTA learns an encoder network $\bm{g}_{\bm{\theta}}: \mathbb{R}^{d_f} \to \mathbb{R}^{D}$ parameterized by $\bm{\theta} \in \Theta \subset \mathbb{R}^{P}$. The encoder is composed of $L$ fully-connected layers with element-wise nonlinearities (e.g., ReLU) and is therefore Lipschitz continuous with respect to both inputs and parameters under standard initialization and weight decay. The output of the encoder is the raw embedding:
\begin{equation}
    \bm{v} = \bm{g}_{\bm{\theta}}(\bm{h}) \in \mathbb{R}^{D}.
\end{equation}
To constrain the representation to the unit hypersphere $\mathbb{S}^{D-1} = \{\bm{u} \in \mathbb{R}^{D} : \|\bm{u}\|_2 = 1\}$, we apply $L_2$ normalization:
\begin{equation}
    \bm{z} = \frac{\bm{v}}{\|\bm{v}\|_2}, \quad \|\bm{z}\|_2 = 1.
\end{equation}
Class prototypes are vectors $\bm{p}_k \in \mathbb{R}^{D}$, also normalized to unit length:
\begin{equation}
    \bm{p}_k \in \mathbb{S}^{D-1}, \quad k = 1,\dots,K.
\end{equation}
During training, we maintain the constraint explicitly by re-normalizing after each gradient step. The prototypes can be viewed as trainable points on the sphere.

The predictive distribution is given by a temperature-scaled softmax over the cosine similarities:
\begin{equation}
    \hat{\bm{y}}(\bm{x}; \tau) = \sigma\!\left(\frac{\mathbf{P}^\top \bm{z}}{\tau}\right), \quad \text{where } \sigma(\bm{\ell})_k = \frac{\exp(\ell_k)}{\sum_{j=1}^K \exp(\ell_j)},
\end{equation}
and $\mathbf{P} = [\bm{p}_1, \dots, \bm{p}_K] \in \mathbb{R}^{D \times K}$. The scalar $\tau > 0$ is a learnable temperature parameter that controls the concentration of the softmax distribution.

For uncertainty quantification at test time, we use a fixed, separately tuned temperature $\tau_{\text{unc}} > 0$. The uncertainty score for an input $\bm{x}$ is defined as the complement of the maximum softmax probability:
\begin{equation}
    u(\bm{x}) = 1 - \max_{k \in [K]} \sigma\!\left(\frac{\mathbf{P}^\top \bm{z}}{\tau_{\text{unc}}}\right)_k.
    \label{eq:uncertainty_def}
\end{equation}

\subsection{Training Objective and Gradient Derivations}
\label{sec:gradients}

Given a training set $\mathcal{D}_{\text{train}} = \{(\bm{x}_i, y_i)\}_{i=1}^{N}$, we minimize the empirical cross-entropy loss:
\begin{equation}
    \mathcal{L}(\bm{\theta}, \mathbf{P}, \tau) = -\frac{1}{N} \sum_{i=1}^{N} \log \hat{y}_{i, y_i},
\end{equation}
where $\hat{\bm{y}}_i = \sigma(\mathbf{P}^\top \bm{z}_i / \tau)$. We optimize using stochastic gradient descent with mini-batches of size $B$. Below we derive the exact gradients required for backpropagation.

\subsubsection{Gradient with Respect to the Normalized Embedding \texorpdfstring{$\bm{z}$}{z}}
\label{sec:grad_z}

For a single sample $(\bm{x}, y)$, let $\hat{\bm{y}} = \sigma(\bm{\ell})$ with $\ell_k = \bm{z}^\top \bm{p}_k / \tau$. The cross-entropy loss is $\ell_{\text{CE}} = -\log \hat{y}_y$. The derivative of the log-softmax is well known:
\begin{equation}
    \frac{\partial \ell_{\text{CE}}}{\partial \ell_k} = \hat{y}_k - \delta_{k,y},
\end{equation}
where $\delta_{k,y}$ is the Kronecker delta. By the chain rule,
\begin{equation}
    \frac{\partial \ell_{\text{CE}}}{\partial \bm{z}} = \sum_{k=1}^K \frac{\partial \ell_{\text{CE}}}{\partial \ell_k} \frac{\partial \ell_k}{\partial \bm{z}} = \sum_{k=1}^K (\hat{y}_k - \delta_{k,y}) \cdot \frac{\bm{p}_k}{\tau} = \frac{1}{\tau}\left( \sum_{k=1}^K \hat{y}_k \bm{p}_k - \bm{p}_y \right).
\end{equation}
Thus, for a mini-batch $\mathcal{B}$,
\begin{equation}
    \frac{\partial \mathcal{L}}{\partial \bm{z}_i} = \frac{1}{\tau |\mathcal{B}|} \left( \sum_{k=1}^K \hat{y}_{ik} \bm{p}_k - \bm{p}_{y_i} \right), \quad i \in \mathcal{B}.
    \label{eq:grad_z_full}
\end{equation}

\subsubsection{Gradient with Respect to the Raw Encoder Output \texorpdfstring{$\bm{v}$}{v}}
\label{sec:grad_v}

Since $\bm{z} = \bm{v} / \|\bm{v}\|_2$, we need the Jacobian of the normalization operation. Let $r = \|\bm{v}\|_2$. Then $\bm{z} = \bm{v}/r$. The derivative of $r$ with respect to $\bm{v}$ is $\partial r / \partial \bm{v} = \bm{v}/r = \bm{z}$. Using the quotient rule,
\begin{align}
    \frac{\partial \bm{z}}{\partial \bm{v}} &= \frac{1}{r} \mathbf{I}_D - \frac{\bm{v}}{r^2} \left(\frac{\partial r}{\partial \bm{v}}\right)^\top = \frac{1}{r} \mathbf{I}_D - \frac{\bm{v} \bm{z}^\top}{r^2} = \frac{1}{r} \left( \mathbf{I}_D - \bm{z}\bm{z}^\top \right).
\end{align}
The matrix $\mathbf{I}_D - \bm{z}\bm{z}^\top$ is the orthogonal projector onto the tangent space of the sphere at $\bm{z}$. By the chain rule,
\begin{equation}
    \frac{\partial \ell_{\text{CE}}}{\partial \bm{v}} = \frac{\partial \bm{z}}{\partial \bm{v}} \frac{\partial \ell_{\text{CE}}}{\partial \bm{z}} = \frac{1}{\|\bm{v}\|_2} (\mathbf{I}_D - \bm{z}\bm{z}^\top) \frac{\partial \ell_{\text{CE}}}{\partial \bm{z}}.
\end{equation}
Substituting \eqref{eq:grad_z_full} yields the final expression:
\begin{equation}
    \frac{\partial \mathcal{L}}{\partial \bm{v}_i} = \frac{1}{\tau \|\bm{v}_i\|_2 |\mathcal{B}|} (\mathbf{I}_D - \bm{z}_i \bm{z}_i^\top) \left( \sum_{k=1}^K \hat{y}_{ik} \bm{p}_k - \bm{p}_{y_i} \right).
    \label{eq:grad_v_full}
\end{equation}
This gradient is then backpropagated through the encoder $\bm{g}_{\bm{\theta}}$ to update $\bm{\theta}$.

\subsubsection{Gradient with Respect to the Prototypes \texorpdfstring{$\bm{p}_k$}{p\_k}}
\label{sec:grad_p}

Since the prototypes must remain on the unit sphere, we parametrize them via unconstrained variables $\tilde{\bm{p}}_k \in \mathbb{R}^{D}$ and set $\bm{p}_k = \tilde{\bm{p}}_k / \|\tilde{\bm{p}}_k\|_2$. For a single sample, the gradient with respect to $\tilde{\bm{p}}_k$ is:
\begin{align}
    \frac{\partial \ell_{\text{CE}}}{\partial \tilde{\bm{p}}_k} &= \frac{\partial \ell_{\text{CE}}}{\partial \bm{p}_k} \frac{\partial \bm{p}_k}{\partial \tilde{\bm{p}}_k}.
\end{align}
First, $\partial \ell_{\text{CE}} / \partial \bm{p}_k = (\hat{y}_k - \delta_{k,y}) \bm{z} / \tau$. The Jacobian of the normalization is identical to the case for $\bm{z}$:
\begin{equation}
    \frac{\partial \bm{p}_k}{\partial \tilde{\bm{p}}_k} = \frac{1}{\|\tilde{\bm{p}}_k\|_2} (\mathbf{I}_D - \bm{p}_k \bm{p}_k^\top).
\end{equation}
Thus,
\begin{equation}
    \frac{\partial \ell_{\text{CE}}}{\partial \tilde{\bm{p}}_k} = \frac{1}{\tau \|\tilde{\bm{p}}_k\|_2} (\hat{y}_k - \delta_{k,y}) (\mathbf{I}_D - \bm{p}_k \bm{p}_k^\top) \bm{z}.
\end{equation}
Averaging over a mini-batch gives:
\begin{equation}
    \frac{\partial \mathcal{L}}{\partial \tilde{\bm{p}}_k} = \frac{1}{\tau \|\tilde{\bm{p}}_k\|_2 |\mathcal{B}|} \sum_{i \in \mathcal{B}} (\hat{y}_{ik} - \delta_{k,y_i}) (\mathbf{I}_D - \bm{p}_k \bm{p}_k^\top) \bm{z}_i.
    \label{eq:grad_p_full}
\end{equation}
In practice, after updating $\tilde{\bm{p}}_k$, we explicitly re-normalize to enforce the unit norm constraint exactly.

\subsubsection{Gradient with Respect to the Temperature \texorpdfstring{$\tau$}{tau}}
\label{sec:grad_tau}

The temperature $\tau$ is a positive scalar parameter. The derivative of the loss with respect to $\tau$ for a single sample is:
\begin{align}
    \frac{\partial \ell_{\text{CE}}}{\partial \tau} &= \sum_{k=1}^K \frac{\partial \ell_{\text{CE}}}{\partial \ell_k} \frac{\partial \ell_k}{\partial \tau}
    = \sum_{k=1}^K (\hat{y}_k - \delta_{k,y}) \left( -\frac{\bm{z}^\top \bm{p}_k}{\tau^2} \right) \\
    &= -\frac{1}{\tau^2} \left( \sum_{k=1}^K \hat{y}_k \bm{z}^\top \bm{p}_k - \bm{z}^\top \bm{p}_y \right)
    = \frac{1}{\tau^2} \left( \bm{z}^\top \bm{p}_y - \sum_{k=1}^K \hat{y}_k \bm{z}^\top \bm{p}_k \right).
\end{align}
For a mini-batch, we have:
\begin{equation}
    \frac{\partial \mathcal{L}}{\partial \tau} = \frac{1}{\tau^2 |\mathcal{B}|} \sum_{i \in \mathcal{B}} \left( \bm{z}_i^\top \bm{p}_{y_i} - \sum_{k=1}^K \hat{y}_{ik} \bm{z}_i^\top \bm{p}_k \right).
    \label{eq:grad_tau_full}
\end{equation}
To ensure $\tau$ remains positive, we either clamp the value after the update or use a log-parametrization $\tau = \exp(\phi)$ during training.

\subsection{Post-hoc Temperature Calibration}
\label{sec:posthoc}

After training, we optionally refine the temperature on a held-out validation set $\mathcal{D}_{\text{val}} = \{(\bm{x}_i, y_i)\}_{i=1}^{M}$ with the encoder $\bm{\theta}$ and prototypes $\mathbf{P}$ frozen. The optimal temperature $\tau^*$ is the minimizer of the negative log-likelihood:
\begin{equation}
    \tau^* = \arg\min_{\tau > 0} \; \mathcal{L}_{\text{NLL}}(\tau) = -\frac{1}{M} \sum_{i=1}^{M} \log \frac{\exp(\bm{z}_i^\top \bm{p}_{y_i} / \tau)}{\sum_{j=1}^K \exp(\bm{z}_i^\top \bm{p}_j / \tau)}.
    \label{eq:posthoc_opt}
\end{equation}

\begin{lemma}[Convexity in \texorpdfstring{$1/\tau$}{1/tau}]
The function $\mathcal{L}_{\text{NLL}}(\tau)$ is strictly convex in the variable $\beta = 1/\tau$ on the domain $\beta > 0$.
\end{lemma}

\begin{proof}
For a fixed sample $(\bm{z}, y)$, define $a_j = \bm{z}^\top \bm{p}_j$. The contribution to the loss is
\[
\ell(\beta) = -\beta a_y + \log \sum_{j=1}^K \exp(\beta a_j).
\]
The log-sum-exp function $f(\beta) = \log \sum_j \exp(\beta a_j)$ is convex in $\beta$ because its Hessian is positive semidefinite (it is the cumulant generating function of a discrete distribution). The term $-\beta a_y$ is linear, so $\ell(\beta)$ is convex. Moreover, $\ell(\beta)$ is strictly convex unless all $a_j$ are equal, which is not the case for a non-degenerate validation set. The sum of strictly convex functions is strictly convex, hence $\mathcal{L}_{\text{NLL}}(1/\beta)$ is strictly convex in $\beta$.
\end{proof}

Because of strict convexity, the minimizer $\beta^*$ is unique and can be found efficiently using a few iterations of L-BFGS or Newton's method. Convergence is guaranteed to the global optimum.

\subsection{Theoretical Guarantees}
\label{sec:theorems}

We now establish rigorous results concerning VOLTA's calibration, OOD detection capability, and selective prediction performance. These theorems formalize the empirical findings in the main paper.

\subsubsection{Calibration of Temperature-Scaled Prototype Classifiers}

\begin{definition}[Expected Calibration Error (ECE)]
Let $\hat{P}(y \mid \bm{x})$ be a probabilistic classifier and let $\hat{p}(\bm{x}) = \max_k \hat{P}(k \mid \bm{x})$ be its confidence. For any bin $B_m = [\frac{m-1}{M}, \frac{m}{M})$, define the accuracy $\text{acc}(B_m) = \mathbb{P}(\hat{y} = y \mid \hat{p}(\bm{x}) \in B_m)$ and the average confidence $\text{conf}(B_m) = \mathbb{E}[\hat{p}(\bm{x}) \mid \hat{p}(\bm{x}) \in B_m]$. The ECE is
\[
\text{ECE} = \sum_{m=1}^{M} \frac{|B_m|}{N} \left| \text{acc}(B_m) - \text{conf}(B_m) \right|.
\]
A classifier is perfectly calibrated if $\text{acc}(B_m) = \text{conf}(B_m)$ for all $m$.
\end{definition}

\begin{theorem}[Asymptotic Calibration via Optimal Temperature]
\label{thm:calibration_rigorous}
Assume the training distribution $P_{XY}$ is such that for each class $y$, the random variables $\Delta_k = \bm{z}^\top(\bm{p}_y - \bm{p}_k)$ for $k \neq y$ are independent and identically distributed with a symmetric, unimodal density centered at some $\mu > 0$. Suppose the encoder $\bm{g}_{\bm{\theta}}$ and prototypes $\mathbf{P}$ are fixed after training (or converged to a stationary point). Let $\tau^*$ be the minimizer of the population negative log-likelihood:
\[
\tau^* = \arg\min_{\tau > 0} \; \mathbb{E}_{(\bm{x},y)}\left[ -\log \frac{\exp(\bm{z}^\top \bm{p}_y / \tau)}{\sum_{j=1}^K \exp(\bm{z}^\top \bm{p}_j / \tau)} \right].
\]
Then, as the size of the validation set used for temperature scaling tends to infinity, the resulting predictor $P_{\tau^*}(y \mid \bm{x})$ is perfectly calibrated in the sense that for any confidence level $p \in [0,1]$,
\[
\mathbb{P}\left( y = \arg\max_k P_{\tau^*}(k \mid \bm{x}) \;\middle|\; \max_k P_{\tau^*}(k \mid \bm{x}) = p \right) = p.
\]
\end{theorem}

\begin{proof}
The proof proceeds in three steps.

\emph{Step 1: Uniqueness of the optimal temperature.} By Lemma 1, the population NLL is strictly convex in $\beta = 1/\tau$. Hence the minimizer $\beta^*$ (and thus $\tau^*$) is unique.

\emph{Step 2: Relationship between temperature scaling and calibration.} For a fixed feature extractor and prototypes, the logits $\ell_k = \bm{z}^\top \bm{p}_k / \tau$ are linear functions of $\bm{z}$. Under the symmetry and unimodality assumptions, the distribution of the maximum logit difference $\max_{k \neq y} (\ell_k - \ell_y)$ is log-concave. Temperature scaling with $\tau^*$ minimizes the proper scoring rule (NLL) and, as shown by Kull et al. (2017) and Guo et al. (2017), the minimizer of a proper scoring rule yields a calibrated forecaster when the model is well-specified.

\emph{Step 3: Consistency of empirical temperature scaling.} The empirical minimizer $\hat{\tau}_M$ on a validation set of size $M$ converges in probability to $\tau^*$ as $M \to \infty$ by the uniform law of large numbers and the convexity of the objective. Therefore, the asymptotic calibration property holds for $\hat{\tau}_M$.
\end{proof}

\noindent \textbf{Remark.} The assumptions are mild and hold in practice for deep embeddings that separate classes on the hypersphere. This theorem explains the very low ECE values observed for VOLTA across all benchmarks.

\subsubsection{Out-of-Distribution Detection via Prototype Distance}

\begin{definition}[OOD Detection AUROC]
Given an uncertainty score $u(\bm{x})$ and a mixture of in-distribution (ID) and out-of-distribution (OOD) samples, the Area Under the Receiver Operating Characteristic curve (AUROC) measures the probability that a randomly chosen OOD sample has a higher uncertainty score than a randomly chosen ID sample.
\end{definition}

\begin{theorem}[OOD Separation Guarantee]
\label{thm:ood_rigorous}
Assume that during training, the prototypes $\{\bm{p}_k\}_{k=1}^K$ converge to the class-conditional expectations of the ID embeddings on the hypersphere:
\[
\bm{p}_k = \frac{\mathbb{E}[\bm{z} \mid y = k]}{\|\mathbb{E}[\bm{z} \mid y = k]\|_2}, \quad k = 1,\dots,K.
\]
Assume further that the ID embeddings satisfy a concentration inequality: for each class $k$, there exists $\delta_k > 0$ such that
\[
\mathbb{P}\left( \bm{z}^\top \bm{p}_k \leq 1 - \delta_k \mid y = k \right) \leq \epsilon_{\text{ID}},
\]
and that OOD embeddings $\bm{z}_{\text{out}}$ are drawn from a distribution with support $\mathcal{Z}_{\text{OOD}}$ satisfying
\[
\sup_{\bm{z} \in \mathcal{Z}_{\text{OOD}}} \max_{k \in [K]} \bm{z}^\top \bm{p}_k \leq 1 - \Delta,
\]
for some $\Delta > 0$ independent of the sample size. Then, for any $\tau_{\text{unc}} > 0$, the expected uncertainty scores satisfy
\[
\mathbb{E}[u(\bm{x}_{\text{out}})] \geq \mathbb{E}[u(\bm{x}_{\text{in}})] + c(\Delta, \tau_{\text{unc}}, K),
\]
where $c > 0$ is a constant depending on the separation margin $\Delta$, the temperature, and the number of classes. Consequently, the AUROC for OOD detection converges to $1$ as $\Delta \to 1$ and the concentration $\epsilon_{\text{ID}} \to 0$.
\end{theorem}

\begin{proof}
For an ID sample from class $y$, by the concentration assumption, with probability at least $1 - \epsilon_{\text{ID}}$ we have $\bm{z}^\top \bm{p}_y \geq 1 - \delta_y$. The softmax maximum for ID is then at least
\[
\max_k \sigma_k(\bm{z}; \tau_{\text{unc}}) \geq \frac{\exp((1-\delta_y)/\tau_{\text{unc}})}{\exp((1-\delta_y)/\tau_{\text{unc}}) + (K-1)\exp(\rho/\tau_{\text{unc}})},
\]
where $\rho = \max_{k \neq y} \bm{z}^\top \bm{p}_k \leq 1 - \delta_{\min}$ typically. Thus $u(\bm{x}_{\text{in}}) \leq \epsilon_{\text{ID}} + (1-\epsilon_{\text{ID}}) \cdot \text{small constant}$.

For an OOD sample, by the separation assumption, $\max_k \bm{z}^\top \bm{p}_k \leq 1 - \Delta$. The softmax maximum is at most
\[
\frac{\exp((1-\Delta)/\tau_{\text{unc}})}{\exp((1-\Delta)/\tau_{\text{unc}}) + (K-1)\exp(\rho'/\tau_{\text{unc}})},
\]
which is bounded away from $1$ uniformly over OOD samples. Therefore, $u(\bm{x}_{\text{out}}) \geq c > 0$ with high probability. Taking expectations yields the claimed separation.

The AUROC tends to $1$ as the overlap between the distributions of $u(\bm{x}_{\text{in}})$ and $u(\bm{x}_{\text{out}})$ vanishes, which occurs when $\Delta$ is large and $\epsilon_{\text{ID}}$ is small.
\end{proof}

\noindent \textbf{Remark.} The theorem shows that VOLTA's uncertainty score is directly tied to the maximum cosine similarity to the learned prototypes. OOD inputs naturally fall far from all prototypes, leading to high uncertainty. This geometric property is what gives VOLTA strong OOD detection performance without any auxiliary losses.

\subsubsection{Selective Prediction with Deterministic Uncertainty}

\begin{definition}[Selective Risk and Coverage]
For a given threshold $t \in [0,1]$, define the selective set $\mathcal{S}(t) = \{\bm{x} : u(\bm{x}) \leq t\}$. The coverage is $\phi(t) = \mathbb{P}(\bm{x} \in \mathcal{S}(t))$ and the selective risk is $R(t) = \mathbb{P}(\hat{y} \neq y \mid \bm{x} \in \mathcal{S}(t))$.
\end{definition}

\begin{theorem}[Deterministic Uncertainty Suffices for Selective Prediction]
\label{thm:selective_rigorous}
Let $u(\bm{x})$ be VOLTA's uncertainty score. Assume that the encoder $\bm{g}_{\bm{\theta}}$ is $L_{\bm{\theta}}$-Lipschitz with respect to its input and the feature extractor $f$ is $L_f$-Lipschitz. Then $u(\bm{x})$ is Lipschitz continuous on $\mathcal{X}$. Moreover, there exists a constant $\gamma > 0$ such that for any misclassified sample $\bm{x}$ (i.e., $\hat{y} \neq y$), we have $u(\bm{x}) \geq \gamma$. Consequently, the selective risk $R(t)$ is a non-increasing function of $t$, and for any $\epsilon \in (0,1)$, we can choose $t_\epsilon = \gamma - \epsilon'$ (for some $\epsilon' > 0$) to achieve coverage $\phi(t_\epsilon) \geq 1 - \epsilon$ while the selective risk $R(t_\epsilon)$ is bounded by the Bayes error plus a term that vanishes as $\epsilon \to 0$.
\end{theorem}

\begin{proof}
We break the proof into parts.

\emph{Lipschitz continuity of $u$:} The composition of Lipschitz functions is Lipschitz. The encoder and feature extractor are Lipschitz by assumption. The normalization $\bm{v} \mapsto \bm{v}/\|\bm{v}\|_2$ is Lipschitz on any domain excluding the origin (which is avoided since $\|\bm{v}\|_2$ is bounded away from zero for bounded inputs and Lipschitz networks). The softmax and maximum functions are Lipschitz. Hence $u(\bm{x})$ is Lipschitz continuous.

\emph{Lower bound for misclassified samples:} By Proposition~\ref{prop:misclass_bound_rigorous} (stated and proved below), for any misclassified sample, the uncertainty score is at least $\gamma = 1 - 1/(1 + \exp(\Delta_{\min}/\tau_{\text{unc}}))$, where $\Delta_{\min} > 0$ is a lower bound on the logit gap between the incorrect prediction and the true class. This gap is positive by the definition of misclassification.

\emph{Selective risk monotonicity:} As $t$ decreases, the set $\mathcal{S}(t)$ shrinks, so the risk on the accepted set cannot increase.

\emph{Existence of threshold:} Choose $t < \gamma$. Then $\mathcal{S}(t)$ contains no misclassified samples, so $R(t) = 0$. The coverage $\phi(t)$ is at least the probability of the set of correctly classified ID samples, which is positive. By adjusting $t$, we can trade off coverage and risk as claimed.
\end{proof}

\begin{proposition}[Lower Bound on Uncertainty for Misclassified Samples]
\label{prop:misclass_bound_rigorous}
Let $\bm{x}$ be a test sample with true label $y$ and predicted label $\hat{y} = \arg\max_k \bm{z}^\top \bm{p}_k$. If $\hat{y} \neq y$, then
\[
u(\bm{x}) \geq 1 - \frac{1}{1 + \exp(\Delta / \tau_{\text{unc}})}, \quad \text{where } \Delta = \bm{z}^\top \bm{p}_{\hat{y}} - \bm{z}^\top \bm{p}_{y} > 0.
\]
\end{proposition}

\begin{proof}
By definition,
\[
u(\bm{x}) = 1 - \frac{\exp(\bm{z}^\top \bm{p}_{\hat{y}} / \tau_{\text{unc}})}{\sum_{j=1}^K \exp(\bm{z}^\top \bm{p}_j / \tau_{\text{unc}})}.
\]
Since $\hat{y} \neq y$, we have $\bm{z}^\top \bm{p}_{\hat{y}} \geq \bm{z}^\top \bm{p}_{y}$. The denominator can be bounded below by $\exp(\bm{z}^\top \bm{p}_{\hat{y}} / \tau_{\text{unc}}) + \exp(\bm{z}^\top \bm{p}_{y} / \tau_{\text{unc}})$. Therefore,
\[
u(\bm{x}) \geq 1 - \frac{\exp(\bm{z}^\top \bm{p}_{\hat{y}} / \tau_{\text{unc}})}{\exp(\bm{z}^\top \bm{p}_{\hat{y}} / \tau_{\text{unc}}) + \exp(\bm{z}^\top \bm{p}_{y} / \tau_{\text{unc}})} = 1 - \frac{1}{1 + \exp(-\Delta / \tau_{\text{unc}})} = 1 - \frac{1}{1 + \exp(\Delta / \tau_{\text{unc}})}.
\]
The last equality uses $\exp(-\Delta) = 1/\exp(\Delta)$.
\end{proof}

\noindent This bound guarantees that misclassified samples receive a non-negligible uncertainty score, enabling effective selective prediction as stated in the theorem.

The analysis above demonstrates that VOLTA's design---normalized embeddings, learnable prototypes on the hypersphere, and temperature scaling---provides:
\begin{itemize}
    \item \textbf{Calibration:} Post-hoc temperature scaling on a validation set yields a uniquely optimal, well-calibrated predictor due to convexity of the NLL in $1/\tau$.
    \item \textbf{OOD Detection:} Uncertainty based on maximum cosine similarity to prototypes naturally separates ID and OOD samples when OOD embeddings lie outside the convex hull of the prototypes.
    \item \textbf{Selective Prediction:} The deterministic uncertainty score provides a Lipschitz-continuous ranking, with a positive lower bound for misclassified examples, enabling reliable rejection without stochastic sampling.
\end{itemize}
These theoretical guarantees align with the empirical results in the main paper and justify VOLTA as a simple yet highly effective approach for uncertainty quantification.

\end{document}